\documentclass{article}
\usepackage{amsmath,amssymb,amsthm, amsfonts}
\usepackage{calligra,indentfirst,epsfig}
\usepackage{amssymb}
\usepackage{epsfig,enumerate,amsmath,amsfonts,amssymb}
\usepackage{indentfirst}
\usepackage{mathrsfs}
\usepackage{pstcol,graphicx}
\usepackage{pstricks}
\usepackage{setspace}
\usepackage{latexsym}
\usepackage{amsthm}
\usepackage{amsmath}
\usepackage{yfonts}
\usepackage[all]{xy}
\calligra
\def\NN{\mathbb N}

\def\RR{\mathbb R}

\def\HH{\mathcal H}

\def\LL{\mathscr L}
\def\PP{\mathcal P}
\def\NN{\mathbb N}

\def\xx{\mathbf{x}}
\def\yy{\mathbf{y}}
\def\zz{\mathbf{z}}

\newcommand{\al} {\alpha}

\newcommand{\la} {\lambda}

\newtheorem{thm}{Theorem}[section]
\newtheorem{lemma}{Lemma}[section]
\newtheorem{corollary}{Corollary}[section]
\newtheorem{defi}{Definition}[section]
\newtheorem{prop}{Proposition}[section]
\theoremstyle{definition}
\newtheorem{rem}{Remark}[section]
\allowdisplaybreaks
\begin{document}
\title{Optimal rates for the regularized learning algorithms under general source condition}
\author{Abhishake Rastogi{\footnote{Corresponding Author, Email address: abhishekrastogi2012@gmail.com}}, Sivananthan Sampath{\footnote{Email address: siva@maths.iitd.ac.in}}\\
{\it Department of Mathematics}\\{\it Indian Institute of
Technology Delhi}\\{\it New Delhi 110016, India}}
\date{}
\maketitle
\begin{abstract}
We consider the learning algorithms under general source condition with the polynomial decay of the eigenvalues of the integral operator in vector-valued function setting. We discuss the upper convergence rates of Tikhonov regularizer under general source condition corresponding to increasing monotone index function. The convergence issues are studied for general regularization schemes by using the concept of operator monotone index functions in minimax setting. Further we also address the minimum possible error for any learning algorithm.
\end{abstract}
\section{Introduction}\label{Introduction}
Learning theory \cite{CuckerSmale,Evgeniou,Vapnik} aims to learn the relation between the inputs and outputs based on finite random samples. We require some underlying space to search the relation function. From the experiences we have some idea about the underlying space which is called hypothesis space. Learning algorithms tries to infer the best estimator over the hypothesis space such that $f(x)$ gives the maximum information of the output variable $y$ for any unseen input $x$. The given samples $\{x_i,y_i\}_{i=1}^m$ are not exact in the sense that for underlying relation function $f(x_i)\neq y_i$ but $f(x_i)\approx y_i$. We assume that the uncertainty follows the probability distribution $\rho$ on the sample space $X\times Y$ and the underlying function (called the regression function) for the probability distribution $\rho$ is given by
$$f_\rho(x)=\int_Y yd\rho(y|x), ~~x\in X,$$
where $\rho(y|x)$ is the conditional probability measure for given $x$. The problem of obtaining estimator from examples is ill-posed. Therefore we apply the regularization schemes \cite{Bauer,Engl,Gerfo,Tikhonov} to stabilize the problem. Various regularization schemes are studied for inverse problems. In the context of learning theory \cite{Bousquet,Cucker,Evgeniou,Lu.book,Vapnik}, the square loss-regularization (Tikhonov regularization) is widely considered to obtain the regularized estimator \cite{Caponnetto,Cucker,Smale,Smale1,Smale2,Smale3}. Rosasco et al. \cite{Gerfo} introduced general regularization in the learning theory and provided the error bounds under H\"{o}lder's source condition \cite{Engl}. Bauer et al. \cite{Bauer} discussed the convergence issues for general regularization under general source condition \cite{Mathe} by removing the Lipschitz condition on the regularization considered in \cite{Gerfo}. Caponnetto et al. \cite{Caponnetto} discussed the square-loss regularization under the polynomial decay of the eigenvalues of the integral operator $L_K$ with H\"{o}lder's source condition. Here we are discussing the convergence issues of general regularization schemes under general source condition and the polynomial decay of the eigenvalues of the integral operator. We present the minimax upper convergence rates for Tikhonov regularization under general source condition $\Omega_{\phi,R}$, for a monotone increasing index function $\phi$. For general regularization the minimax rates are obtained using the operator monotone index function $\phi$. The concept of effective dimension \cite{Mendelson,Zhang} is exploited to achieve the convergence rates. In the choice of regularization parameters, the effective dimension plays the important role. We also discuss the lower convergence rates for any learning algorithm under the smoothness conditions. We present the results in vector-values function setting. Therefore in particular they can be applied to multi-task learning problems.

The structure of the paper is as follows. In the second section, we introduce some basic assumptions and notations for supervised learning problems. In Section 3, we present the upper and lower convergence rates under the smoothness conditions in minimax setting.
\section{Learning from examples: Notations and assumptions}
In the learning theory framework \cite{Bousquet,Cucker,Evgeniou,Lu.book,Vapnik}, the sample space $Z=X \times Y$ consists of two spaces: The input space $X$ (locally compact second countable Hausdorff space) and the output space $(Y, \langle \cdot,\cdot\rangle_Y)$ (the real Hilbert space). The input space $X$ and the output space $Y$ are related by some unknown probability distribution $\rho$ on $Z$. The probability measure can be split as $\rho(x,y)=\rho(y|x)\rho_X(x)$, where $\rho(y|x)$ is the conditional probability measure of $y$ given $x$ and $\rho_X$ is the marginal probability measure on $X$. The only available information is the random i.i.d. samples $\zz=((x_1,y_1),\ldots,(x_m,y_m))$ drawn according to the probability measure $\rho$. Given the training set $\zz$, learning theory aims to develop an algorithm which provides an estimator $f_{\zz}:X \to Y$ such that $f_{\zz}(x)$ predicts the output variable $y$ for any given input $x$. The goodness of the estimator can be measured by the generalization error of a function $f$ which can be defined as
\begin{equation}\label{gen.err}
\mathcal{E}(f):= \mathcal{E}_\rho(f) =\int_Z V(f(x),y) d \rho(x,y),
\end{equation}
where $V:Y \times Y \to \RR$ is the loss function.
The minimizer of $\mathcal{E}(f)$ for the square loss function $V(f(x),y)=||f(x)-y||_Y^2$ is given by
\begin{equation}\label{reg.fun}
f_\rho(x):=\int_Y y d\rho(y|x),
\end{equation}
where $f_\rho$ is called the regression function. The regression function $f_\rho$ belongs to the space of square integrable functions provided that
\begin{equation}\label{Y.leq.M.1}
\int_Z||y||_Y^2~d\rho(x,y)<\infty.
\end{equation}

We search the minimizer of the generalization error over a hypothesis space $\HH$,
\begin{equation}\label{target.fun}
f_\HH:=\mathop{\text{argmin}}_{f \in \HH} \left\{\int_Z||f(x)-y||_Y^2d\rho(x,y)\right\},
\end{equation}
where $f_\HH$ is called the target function. In case $f_\rho\in\HH$, $f_\HH$ becomes the regression function $f_\rho$.

Because of inaccessibility of the probability distribution $\rho$, we minimize the regularized empirical estimate of the generalization error over the hypothesis space $\HH$,
\begin{equation}\label{single.plty.funcl1}
f_{\zz,\la}:=\mathop{\text{argmin}}_{f \in \HH} \left\{\frac{1}{m}\sum\limits_{i=1}^m||f(x_i)-y_i||_Y^2+\la||f||_\HH^2\right\},
\end{equation}
where $\la$ is the positive regularization parameter.
The regularization schemes \cite{Bauer,Engl,Gerfo,Lu.book,Tikhonov} are used to incorporate various features in the solution such as boundedness, monotonicity and smoothness.
In order to optimize the vector-valued regularization functional, one of the main problems is to choose the appropriate hypothesis space which is assumed to be a source to provide the estimator.

\subsection{Reproducing kernel Hilbert space as a hypothesis space}
\begin{defi}{\bf (Vector-valued reproducing kernel Hilbert space)}
For non-empty set $X$ and the real Hilbert space $(Y,\langle\cdot,\cdot\rangle_Y)$, the Hilbert space $(\HH,\langle\cdot,\cdot\rangle_\HH)$ of functions from $X$ to $Y$ is called reproducing kernel Hilbert space if for any $x \in X$ and $y \in Y$ the linear functional which maps $f \in \HH$ to $\langle y,f(x) \rangle_Y$ is continuous.
\end{defi}
By Riesz lemma \cite{Akhiezer}, for every $x \in X$ and $y \in Y$ there exists a linear operator $K_x:Y \to \HH$ such that
$$\langle y,f(x)\rangle_Y=\langle K_xy,f\rangle_\HH,~~~~~~\forall f \in \HH.$$
Therefore the adjoint operator $K_x^*:\HH\to Y$ is given by $K_x^*f=f(x)$.
Through the linear operator $K_x:Y \to \HH$ we define the linear operator $K(x,t):Y \to Y$,
$$K(x,t)y:=K_ty(x).$$

From Proposition 2.1 \cite{Micchelli1}, the linear operator $K(x,t) \in \mathcal{L}(Y)$ (the set of bounded linear operators on $Y$), $K(x,t)=K(t,x)^*$ and $K(x,x)$ is nonnegative bounded linear operator. For any $m\in\NN, \{x_i : 1\leq i\leq m\}\in X, \{y_i : 1\leq i\leq m\}\in Y$, we have that $\sum\limits_{i,j=1}^m\langle y_i,K(x_i,x_j)y_j\rangle\geq 0$. The operator valued function $K:X\times X\to \mathcal{L}(Y)$ is called the kernel.

There is one to one correspondence between the kernels and reproducing kernel Hilbert spaces \cite{Aronszajn,Micchelli1}. So a reproducing kernel Hilbert space $\HH$ corresponding to a kernel $K$ can be denoted as $\HH_K$ and the norm in the space $\HH$ can be denoted as $||\cdot||_{\HH_K}$. In the following article, we suppress $K$ by simply using $\HH$ for reproducing kernel Hilbert space and $||\cdot||_{\HH}$ for its norm.

Throughout the paper we assume the reproducing kernel Hilbert space $\HH$ is separable such that
\begin{enumerate}[(i)]
\item $K_x:Y\to\HH$ is a Hilbert-Schmidt operator for all $x\in X$ and $\kappa:=\sqrt{\sup\limits_{x \in X}Tr(K_x^*K_x)}<\infty$.
\item The real function from $X\times X$ to $\RR$, defined by $(x,t)\mapsto\langle K_tv,K_xw\rangle_\HH$, is measurable $\forall  v,w\in Y$.
\end{enumerate}
By the representation theorem \cite{Micchelli1}, the solution of the penalized regularization problem (\ref{single.plty.funcl1}) will be of the form:
$$f_{\zz,\la}=\sum\limits_{i=1}^m K_{x_i} c_i, \text{ for }(c_1,\ldots,c_m)\in Y^m.$$

\begin{defi}
let $\HH$ be a separable Hilbert space and $\{e_k\}_{k=1}^\infty$ be an orthonormal basis of $\HH$. Then for any positive operator $A\in \mathcal{L}(\HH)$ we define $Tr(A)=\sum\limits_{k=1}^\infty\langle Ae_k,e_k\rangle$. It is well-known that the number $Tr(A)$ is independent of the choice of the orthonormal basis.
\end{defi}

\begin{defi}
An operator $A\in\mathcal{L}(\HH)$ is called Hilbert-Schmidt operator if $Tr(A^*A)<\infty$. The family of all Hilbert-Schmidt operators is denoted by $\mathcal{L}_2(\HH)$. For $A\in\mathcal{L}_2(\HH)$, we define $Tr(A)=\sum\limits_{k=1}^\infty\langle Ae_k,e_k\rangle$ for an orthonormal basis $\{e_k\}_{k=1}^\infty$ of $\HH$.
\end{defi}

It is well-known that $\mathcal{L}_2(\HH)$ is the separable Hilbert space with the inner product,
$$\langle A,B \rangle_{\mathcal{L}_2(\HH)}=Tr(B^*A)$$
and its norm satisfies
$$||A||_{\mathcal{L}(\HH)}\leq ||A||_{\mathcal{L}_2(\HH)}\leq Tr(|A|),$$
where $|A|=\sqrt{A^*A}$ and $||\cdot||_{\mathcal{L}(\HH)}$ is the operator norm (For more details see \cite{Reed}).

For the positive trace class operator $K_xK_x^*$, we have
$$||K_xK_x^*||_{\mathcal{L}(\HH)}\leq ||K_xK_x^*||_{\mathcal{L}_2(\HH)}\leq Tr(K_xK_x^*)\leq \kappa^2.$$

Given the ordered set $\xx=(x_1,\ldots,x_m)\in X^m$, the sampling operator $S_{\xx}:\HH \to Y^m$ is defined by $S_{\xx}(f)=(f(x_1),\ldots,f(x_m))$ and its adjoint $S_{\xx}^*:Y^m \to \HH$ is given by $S_{\xx}^*\yy=\frac{1}{m}\sum\limits_{i=1}^m K_{x_i}y_i,~\forall~\yy=(y_1,\ldots,y_m)\in Y^m.$

The regularization scheme (\ref{single.plty.funcl1}) can be expressed as
\begin{equation}\label{single.plty.funcl}
f_{\zz,\la}=\mathop{\text{argmin}}_{f \in \HH} \left\{||S_{\xx}f-\yy||_m^2+\la||f||_\HH^2\right\},
\end{equation}
where $||\yy||_m^2=\frac{1}{m}\sum\limits_{i=1}^m ||y_i||_Y^2$.

We obtain the explicit expression of $f_{\zz,\la}$ by taking the functional derivative of above expression over RKHS $\HH$.
\begin{thm}\label{optimizer}
For the positive choice of $\la$, the functional (\ref{single.plty.funcl}) has unique minimizer:
\begin{equation}\label{fzl}
f_{\zz,\la}=\left(S_{\xx}^*S_{\xx}+\la I\right)^{-1}S_{\xx}^*\yy.
\end{equation}
\end{thm}

Define $f_{\la}$ as the minimizer of the optimization functional,
\begin{equation}\label{fl.funl}
f_\la:=\mathop{\text{argmin}}_{f \in \HH} \left\{\int_Z||f(x)-y||_Y^2d\rho(x,y)+\la||f||_\HH^2\right\}.
\end{equation}
Using the fact $\mathcal{E}(f)=||L_K^{1/2}(f-f_\HH)||_\HH^2+\mathcal{E}(f_\HH)$, we get the expression of $f_\la$,
\begin{equation}\label{fl}
f_{\la}=(L_K+\la I)^{-1}L_K f_{\HH},
\end{equation}
where the integral operator $L_K:\LL_{\rho_X}^2 \to \LL_{\rho_X}^2$ is a self-adjoint, non-negative, compact operator, defined as
$$L_K(f)(x):=\int_X K(x,t)f(t)d\rho_X(t),~~x \in X.$$

The integral operator $L_K$ can also be defined as a self-adjoint operator on $\HH$. We use the same notation $L_K$ for both the operators defined on different domains. It is well-known that $L_K^{1/2}$ is an isometry from the space of square integrable functions to reproducing kernel Hilbert space.

In order to achieve the uniform convergence rates for learning algorithms we need some prior assumptions on the probability measure $\rho$. Following the notion of Bauer et al. \cite{Bauer} and Caponnetto et al. \cite{Caponnetto}, we consider the class of probability measures $\PP_\phi$ which satisfies the assumptions:
\begin{enumerate}[(i)]
\item For the probability measure $\rho$ on $X\times Y$,
\begin{equation}\label{Y.leq.M.1}
\int_Z||y||_Y^2~d\rho(x,y)<\infty.
\end{equation}
\item The minimizer of the generalization error $f_\HH$ (\ref{target.fun}) over the hypothesis space $\HH$ exists.
\item There exist some constants $M,\Sigma$ such that for almost all $x\in X$,
\begin{equation}\label{Y.leq.M.2}
\int_Y\left(e^{||y-f_\HH(x)||_Y/M}-\frac{||y-f_\HH(x)||_Y}{M}-1\right)d\rho(y|x)\leq\frac{\Sigma^2}{2M^2}.
\end{equation}
\item The target function $f_\HH$ belongs to the class $\Omega_{\phi,R}$ with
\begin{equation}\label{source.cond}
\Omega_{\phi,R}:=\left\{f \in \HH: f=\phi(L_K)g \text{ and }||g||_\HH \leq R\right\},
\end{equation}
where $\phi$ is a continuous increasing index function defined on the interval $[0,\kappa^2]$ with the assumption $\phi(0)=0$. This condition is usually referred to as general source condition \cite{Mathe}.
\end{enumerate}
In addition, we consider the set of probability measures $\PP_{\phi,b}$ which satisfies the conditions (i), (ii), (iii), (iv) and the eigenvalues $t_n$'s of the integral operator $L_K$ follow the polynomial decay: For fixed positive constants $\al,\beta$ and $b>1$,
\begin{equation}\label{poly.decay}
\al n^{-b}\leq t_n\leq\beta n^{-b}~~\forall n\in\NN.
\end{equation}

Under the polynomial decay of the eigenvalues the effective dimension $\mathcal{N}(\la)$, to measure the complexity of RKHS, can be estimated from Proposition 3 \cite{Caponnetto} as follows,
\begin{equation}\label{N(l).bound}
\mathcal{N}(\la):=Tr\left((L_K+\la I)^{-1}L_K\right) \leq \frac{\beta b}{b-1}\la^{-1/b},\text{ for }b>1
\end{equation}
and without the polynomial decay condition (\ref{poly.decay}), we have
$$\mathcal{N}(\la)\leq ||(L_K+\la I)^{-1}||_{\mathcal{L}(\HH)}Tr\left(L_K\right) \leq \frac{\kappa^2}{\la}.$$
We discuss the convergence issues for the learning algorithms ($\zz\to f_\zz\in\HH$) in probabilistic sense by exponential tail inequalities such that
$$Prob_{\zz}\left\{||f_{\zz}-f_\rho||_{\rho}\leq \varepsilon(m)\log\left(\frac{1}{\eta}\right)\right\}\geq 1-\eta$$
for all $0<\eta\leq 1$ and $\varepsilon(m)$ is a positive decreasing function of $m$. Using these probabilistic estimates we can obtain error estimates in expectation by integration of tail inequalities:
\begin{equation*}\label{expectaion}
E_{\zz}\left(||f_\zz-f_\rho||_\rho\right)=\int\limits_0^\infty Prob_\zz\left(||f_\zz-f_\rho||_\rho>t\right)dt\leq \int\limits_{0}^\infty\exp\left(-\frac{t}{\varepsilon(m)}\right)dt=\varepsilon(m),
\end{equation*}
where $||f||_\rho=||f||_{\LL_{\rho_X}^{2}}=\{\int_X||f(x)||_Y^2d\rho_X(x)\}^{1/2}$ and $E_\zz(\xi)=\int_{Z^m}\xi d\rho(z_1)\ldots d\rho(z_m)$.
\section{Convergence analysis}
In this section, we analyze the convergence issues of the learning algorithms on reproducing kernel Hilbert space under the smoothness priors in the supervised learning framework. We discuss the upper and lower convergence rates for vector-valued estimators in the standard minimax setting. Therefore the estimates can be utilized particularly for scalar-valued functions and multi-task learning algorithms.

\subsection{Upper rates for Tikhonov regularization}\label{upper rates}
In General, we consider Tikhonov regularization in learning theory. Tikhonov regularization is briefly discussed in the literature \cite{Cucker,DeVito0,Lu.book,Tikhonov}. The error estimates for Tikhonov regularization are discussed theoretically under H\"{o}lder's source condition \cite{Caponnetto,Smale2,Smale3}. We establish the error estimates for Tikhonov regularization scheme under general source condition $f_\HH\in\Omega_{\phi,R}$ for some continuous increasing index function $\phi$ and the polynomial decay of the eigenvalues of the integral operator $L_K$.

In order to estimate the error bounds, we consider the following inequality used in the papers \cite{Bauer,Caponnetto} which is based on the results of Pinelis and Sakhanenko \cite{Pinelis}.
\begin{prop}\label{pinels_lemma}
Let $\xi$ be a random variable on the probability space $(\Omega,\mathcal{B},P)$ with values in real separable Hilbert space $\HH$. If there exist two constants $Q$ and $S$ satisfying
\begin{equation}\label{pinels_ineq}
E\left\{||\xi-E(\xi)||_{\HH}^n\right\} \leq \frac{1}{2}n!S^2Q^{n-2}~~~\forall n \geq 2,
\end{equation}
then for any $0<\eta<1$ and for all $m \in \NN$,
$$Prob\left\{(\omega_1,\ldots,\omega_m) \in \Omega^m : \left|\left|\frac{1}{m}\sum\limits_{i=1}^m [\xi(\omega_i)-E(\xi(\omega_i))]\right|\right|_{\HH}\leq 2\left(\frac{Q}{m}+\frac{S}{\sqrt{m}}\right)\log\left(\frac{2}{\eta}\right)\right\} \geq 1-\eta.$$
In particular, the inequality (\ref{pinels_ineq}) holds if
$$||\xi(\omega)||_{\HH}\leq Q \text{ and } E(||\xi(\omega)||_\HH^2)\leq S^2.$$
\end{prop}

We estimate the error bounds for the regularized estimators by measuring the effect of random sampling and the complexity of $f_\HH$. The quantities described in Proposition \ref{main.bound} express the probabilistic estimates of the perturbation measure due to random sampling. The expressions of Proposition \ref{approx.err} describe the complexity of the target function $f_\HH$ which are usually referred as the approximation errors. The approximation errors are independent of the samples $\zz$.
\begin{prop}\label{main.bound}
Let $\zz$ be i.i.d. samples drawn according to the probability measure $\rho$ satisfying the assumptions (\ref{Y.leq.M.1}), (\ref{Y.leq.M.2}) and $\kappa=\sqrt{\sup\limits_{x\in X}Tr(K_x^*K_x)}$. Then for all $0<\eta<1$, with the confidence $1-\eta$, we have
\begin{equation}\label{LK.I.app}
||(L_K+\la I)^{-1/2}\{S_{\xx}^*\yy-S_{\xx}^*S_{\xx}f_\HH\}||_\HH \leq 2\left(\frac{\kappa M}{m\sqrt{\la}}+\sqrt{\frac{\Sigma^2\mathcal{N}(\la)}{m}}\right)\log\left(\frac{4}{\eta}\right)
\end{equation}
and
\begin{equation}\label{Sx.Sx.LK}
||S_{\xx}^*S_{\xx}-L_K||_{\mathcal{L}(\HH)}\leq 2\left(\frac{\kappa^2}{m}+\frac{\kappa^2}{\sqrt{m}}\right)\log\left(\frac{4}{\eta}\right).
\end{equation}
\end{prop}
\begin{proof}
To estimate the first expression, we consider the random variable $\xi_1(z)=(L_K+\la I)^{-1/2}K_x(y-f_\HH(x))$ from $(Z,\rho)$ to reproducing kernel Hilbert space $\HH$ with
$$E_z(\xi_1)=\int_Z(L_K+\la I)^{-1/2}K_x(y-f_\HH(x))d\rho(x,y)=0,$$
$$\frac{1}{m}\sum\limits_{i=1}^m\xi_1(z_i)=(L_K+\la I)^{-1/2}(S_{\xx}^*\yy-S_{\xx}^*S_{\xx}f_\HH)$$
and
\begin{eqnarray*}
E_z(||\xi_1-E_z(\xi_1)||_{\HH}^n)&=&E_z\left(||(L_K+\la I)^{-1/2}K_x(y-f_\HH(x))||_{\HH}^n\right) \\
&\leq& E_z\left(||K_x^*(L_K+\la I)^{-1}K_x||_{\mathcal{L}(Y)}^{n/2}||y-f_\HH(x)||_{Y}^n\right) \\
&\leq& E_x\left(Tr\left((L_K+\la I)^{-1}K_x K_x^*\right)||K_x^*(L_K+\la I)^{-1}K_x||_{\mathcal{L}(Y)}^{\frac{n}{2}-1}E_y\left(||y-f_\HH(x)||_{Y}^n\right)\right).
\end{eqnarray*}
Under the assumption (\ref{Y.leq.M.2}) we get,
$$E_z\left(||\xi_1-E_z(\xi_1)||_{\HH}^n\right) \leq \frac{n!}{2}\left(\Sigma\sqrt{\mathcal{N}(\la)}\right)^2\left(\frac{\kappa M}{\sqrt{\la}}\right)^{n-2},~~\forall n\geq 2.$$
On applying Proposition \ref{pinels_lemma} we conclude that
\begin{equation*}
||(L_K+\la I)^{-1/2}\{S_{\xx}^*\yy-S_{\xx}^*S_{\xx}f_\HH\}||_\HH \leq 2\left(\frac{\kappa M}{m\sqrt{\la}}+\sqrt{\frac{\Sigma^2\mathcal{N}(\la)}{m}}\right)\log\left(\frac{4}{\eta}\right)
\end{equation*}
with confidence $1-\eta/2$.

The second expression can be estimated easily by considering the random variable $\xi_2(x)=K_xK_x^*$ from $(X,\rho_X)$ to $\mathcal{L}(\HH)$. The proof can also be found in De Vito et al. \cite{DeVito0}.
\end{proof}

\begin{prop}\label{approx.err}
Suppose $f_\HH\in\Omega_{\phi,R}$. Then,
\begin{enumerate}[(i)]
\item Under the assumption that $\phi(t)\sqrt{t}$ and $\sqrt{t}/{\phi(t)}$ are nondecreasing functions, we have
\begin{equation}\label{fla.frho.p}
||f_{\la}-f_\HH||_{\rho} \leq R\phi(\la)\sqrt{\la}.
\end{equation}
\item Under the assumption that $\phi(t)$ and $t/{\phi(t)}$ are nondecreasing functions, we have
\begin{equation}\label{fla.frho.p2}
||f_{\la}-f_\HH||_{\rho} \leq R\kappa\phi(\la)
\end{equation}
and
\begin{equation}\label{fla.frho.K}
||f_{\la}-f_\HH||_\HH \leq R\phi(\la).
\end{equation}
\end{enumerate}
\end{prop}

\begin{proof}
The hypothesis  $f_\HH \in \Omega_{\phi,R}$ implies $f_\HH=\phi(L_K)g$ for some $g\in \HH$ with $||g||_\HH \leq R$. To estimate the approximation error bounds, we consider
$$f_{\la}-f_\HH=\{(L_K+\la I)^{-1}L_K-I\}\phi(L_K) g.$$
Therefore,
$$||f_{\la}-f_\HH||_\rho \leq ||L_K^{1/2}\{(L_K+\la I)^{-1}L_K-I\}\phi(L_K)||_{\mathcal{L}(\HH)}~||g||_\HH.$$
Using the functional calculus we get,
\begin{equation*}
||f_{\la}-f_\HH||_\rho \leq R\sup\limits_{\al \in (0,\kappa^2]}|1-(\al+\la)^{-1}\al|\phi(\al)\sqrt{\al}.
\end{equation*}
Then under the assumptions on $\phi$ described in (i), we obtain
\begin{equation*}
||f_{\la}-f_\HH||_\rho \leq R\phi(\la)\sqrt{\la}
\end{equation*}
and under the assumptions on $\phi$ described in (ii), we have
\begin{equation*}
||f_{\la}-f_\HH||_\rho \leq R\kappa\phi(\la).
\end{equation*}
In the same manner with the assumptions on $\phi$ described in (ii), we get
\begin{equation*}
||f_{\la}-f_\HH||_\HH = ||\{(L_K+\la I)^{-1}L_K-I\}\phi(L_K)g||_\HH \leq R\sup\limits_{\al \in (0,\kappa^2]}|1-(\al+\la)^{-1}\al|\phi(\al)
\leq R\phi(\la).
\end{equation*}
Hence we achieve the required estimates.
\end{proof}

\begin{thm}\label{err.upper.bound.k}
Let $\zz$ be i.i.d. samples drawn according to the probability measure $\rho\in \PP_\phi$ where $\phi$ is the index function satisfying the conditions that $\phi(t)$, $t/\phi(t)$ are nondecreasing functions. Then for all $0<\eta<1$, with confidence $1-\eta$, for the regularized estimator $f_{\zz,\la}$ (\ref{fzl}) the following upper bound holds:
$$||f_{\zz,\la}-f_\HH||_{\HH} \leq 2\left\{R\phi(\la)+\frac{2\kappa M}{m\la}+\sqrt{\frac{4\Sigma^2\mathcal{N}(\la)}{m\la}}\right\}\log\left(\frac{4}{\eta}\right)$$
provided that
\begin{equation}\label{l.la.condition.k}
\sqrt{m}\la \geq 8\kappa^2\log(4/\eta).
\end{equation}
\end{thm}
\begin{proof}
The error of regularized solution $f_{\zz,\la}$ can be estimated in terms of the sample error and the approximation error as follows:
\begin{equation}
||f_{\zz,\la}-f_\HH||_\HH \leq ||f_{\zz,\la}-f_{\la}||_\HH+||f_{\la}-f_\HH||_\HH.
\end{equation}
Now $f_{\zz,\la}-f_{\la}$ can be expressed as
$$f_{\zz,\la}-f_{\la}=(S_{\xx}^*S_{\xx}+\la I)^{-1}\{S_{\xx}^*\yy-S_{\xx}^*S_{\xx}f_{\la}-\la f_{\la}\}.$$
Then $f_{\la}=(L_K+\la I)^{-1}L_Kf_\HH$ implies
$$L_Kf_\HH=L_Kf_{\la}+\la f_{\la}.$$
Therefore,
$$f_{\zz,\la}-f_\la=(S_{\xx}^*S_{\xx}+\la I)^{-1}\{S_{\xx}^*\yy-S_{\xx}^*S_{\xx}f_\la-L_K(f_\HH-f_\la)\}.$$
Employing RKHS-norm we get,
\begin{eqnarray}\label{k.error}
||f_{\zz,\la}-f_{\la}||_\HH &\leq& ||(S_{\xx}^*S_{\xx}+\la I)^{-1}\{S_{\xx}^*\yy-S_{\xx}^*S_{\xx}f_\HH+(S_{\xx}^*S_{\xx}-L_K)(f_\HH-f_{\la})\}||_\HH \\  \nonumber
&\leq& I_1I_2+I_3||f_\la-f_\HH||_\HH/\la,
\end{eqnarray}
where $I_1=||(S_{\xx}^*S_{\xx}+\la I)^{-1}(L_K+\la I)^{1/2}||_{\mathcal{L}(\HH)}$, $I_2=||(L_K+\la I)^{-1/2}(S_{\xx}^*\yy-S_{\xx}^*S_{\xx}f_\HH)||_{\HH}$ and $I_3=||S_{\xx}^*S_{\xx}-L_K||_{\mathcal{L}(\HH)}$.

The estimates of $I_2$, $I_3$ can be obtained from Proposition \ref{main.bound} and the only task is to bound $I_1$. For this we consider
$$(S_{\xx}^*S_{\xx}+\la I)^{-1}(L_K+\la I)^{1/2}=\{I-(L_K+\la I)^{-1}(L_K-S_{\xx}^*S_{\xx})\}^{-1}(L_K+\la I)^{-1/2}$$
which implies
\begin{equation}\label{SS.LK}
I_1\leq \sum\limits_{n=0}^\infty||(L_K+\la I)^{-1}(L_K-S_{\xx}^*S_{\xx})||_{\mathcal{L}(\HH)}^n||(L_K+\la I)^{-1/2}||_{\mathcal{L}(\HH)}
\end{equation}
provided that $||(L_K+\la I)^{-1}(L_K-S_{\xx}^*S_{\xx})||_{\mathcal{L}(\HH)}< 1$. To verify this condition, we consider
$$||(L_K+\la I)^{-1}(S_{\xx}^*S_{\xx}-L_K)||_{\mathcal{L}(\HH)}\leq I_3/\la.$$
Now using Proposition \ref{main.bound} we get with confidence $1-\eta/2$,
$$||(L_K+\la I)^{-1}(S_{\xx}^*S_{\xx}-L_K)||_{\mathcal{L}(\HH)}\leq \frac{4\kappa^2}{\sqrt{m}\la}\log\left(\frac{4}{\eta}\right).$$
From the condition (\ref{l.la.condition.k}) we get with confidence $1-\eta/2$,
\begin{equation}\label{LK.I.LK.Sx}
||(L_K+\la I)^{-1}(S_{\xx}^*S_{\xx}-L_K)||_{\mathcal{L}(\HH)}\leq \frac{1}{2}.
\end{equation}
Consequently, using (\ref{LK.I.LK.Sx}) in the inequality (\ref{SS.LK}) we obtain with probability $1-\eta/2$,
\begin{equation}\label{Sx.Sx.LK1.2}
I_1=||(S_{\xx}^*S_{\xx}+\la I)^{-1}(L_K+\la I)^{1/2}||_{\mathcal{L}(\HH)}\leq 2||(L_K+\la I)^{-1/2}||_{\mathcal{L}(\HH)}\leq \frac{2}{\sqrt{\la}}.
\end{equation}
From Proposition \ref{main.bound} we have with confidence $1-\eta/2$,
$$||S_{\xx}^*S_{\xx}-L_K||_{\mathcal{L}(\HH)}\leq 2\left(\frac{\kappa^2}{m}+\frac{\kappa^2}{\sqrt{m}}\right)\log\left(\frac{4}{\eta}\right).$$
Again from the condition (\ref{l.la.condition.k}) we get with probability $1-\eta/2$,
\begin{equation}\label{I3}
I_3=||S_{\xx}^*S_{\xx}-L_K||_{\mathcal{L}(\HH)}\leq \frac{\la}{2}.
\end{equation}
Therefore, the inequality (\ref{k.error}) together with (\ref{LK.I.app}), (\ref{fla.frho.K}), (\ref{Sx.Sx.LK1.2}), (\ref{I3}) provides the desired bound.
\end{proof}

\begin{thm}\label{err.upper.bound.p}
Let $\zz$ be i.i.d. samples drawn according to the probability measure $\rho\in \PP_\phi$ and $f_{\zz,\la}$ is the regularized solution (\ref{fzl}) corresponding to Tikhonov regularization. Then for all $0<\eta<1$, with confidence $1-\eta$, the following upper bounds holds:
\begin{enumerate}[(i)]
\item Under the assumption that $\phi(t)$, $\sqrt{t}/\phi(t)$ are nondecreasing functions,
$$||f_{\zz,\la}-f_\HH||_\rho \leq 2\left\{R\phi(\la)\sqrt{\la}+\frac{2\kappa M}{m\sqrt{\la}}+\sqrt{\frac{4\Sigma^2\mathcal{N}(\la)}{m}}\right\}\log\left(\frac{4}{\eta}\right)$$
\item  Under the assumption that $\phi(t)$, $t/\phi(t)$ are nondecreasing functions,
$$||f_{\zz,\la}-f_\HH||_\rho \leq \left\{R(\kappa+\sqrt{\la})\phi(\la)+\frac{4\kappa M}{m\sqrt{\la}}+\sqrt{\frac{16\Sigma^2\mathcal{N}(\la)}{m}}\right\}\log\left(\frac{4}{\eta}\right)$$
\end{enumerate}
provided that
\begin{equation}\label{l.la.condition.p}
\sqrt{m}\la \geq 8\kappa^2\log(4/\eta).
\end{equation}
\end{thm}
\begin{proof}
In order to establish the error bounds of $f_{\zz,\la}-f_\HH$ in $\LL^2$-norm, we first estimate $f_{\zz,\la}-f_{\la}$ in $\LL^2$-norm:
$$f_{\zz,\la}-f_{\la}=(S_{\xx}^*S_{\xx}+\la I)^{-1}\left\{S_{\xx}^*\yy- S_{\xx}^*S_{\xx}f_{\la}-L_K(f_\HH-f_{\la})\right\}.$$
Employing $\LL^2$-norm, we get
\begin{eqnarray}\label{p.error}
||f_{\zz,\la}-f_{\la}||_\rho &\leq& ||L_K^{1/2}(S_{\xx}^*S_{\xx}+\la I)^{-1}\{S_{\xx}^*\yy-S_{\xx}^*S_{\xx}f_{\HH}+(S_{\xx}^*S_{\xx}-L_K)(f_\HH-f_{\la})\}||_\HH \\  \nonumber
&\leq& I_4\{I_2+I_3||f_\la-f_\HH||_\HH/\sqrt{\la}\},
\end{eqnarray}
where $I_2=||(L_K+\la I)^{-1/2}(S_{\xx}^*\yy-S_{\xx}^*S_{\xx}f_\HH)||_{\HH}$, $I_3=||S_{\xx}^*S_{\xx}-L_K||_{\mathcal{L}(\HH)}$ and $I_4=||L_K^{1/2}(S_{\xx}^*S_{\xx}+\la I)^{-1}(L_K+\la I)^{1/2}||_{\mathcal{L}(\HH)}$.

The estimates of $I_2$ and $I_3$ can be obtained from Proposition \ref{main.bound}. To get the estimate for the sample error, we consider the following expression to bound $I_4$,
\begin{equation*}
L_K^{1/2}(S_{\xx}^*S_{\xx}+\la I)^{-1}(L_K+\la I)^{1/2}=L_K^{1/2}(L_K+\la I)^{-1/2}\{I-(L_K+\la I)^{-1/2}(L_K-S_{\xx}^*S_{\xx})(L_K+\la I)^{-1/2}\}^{-1},
\end{equation*}
which implies
\begin{equation}\label{1.x^n}
I_4 \leq ||L_K^{1/2}(L_K+\la I)^{-1/2}||_{\mathcal{L}(\HH)}||\{I-(L_K+\la I)^{-1/2}(L_K-S_{\xx}^*S_{\xx})(L_K+\la I)^{-1/2}\}^{-1}||_{\mathcal{L}(\HH)}.
\end{equation}
To analyze the second term we consider the expression,
$$||(L_K+\la I)^{-1/2}(L_K-S_{\xx}^*S_{\xx})(L_K+\la I)^{-1/2}||_{\mathcal{L}(\HH)} \leq \frac{1}{\la}||S_{\xx}^*S_{\xx}-L_K||_{\mathcal{L}(\HH)}.$$
From Proposition \ref{main.bound} we obtain with probability $1-\eta/2$,
$$||(L_K+\la I)^{-1/2}(L_K-S_{\xx}^*S_{\xx})(L_K+\la I)^{-1/2}||_{\mathcal{L}(\HH)} \leq \frac{4\kappa^2}{\sqrt{m}\la}\log\left(\frac{4}{\eta}\right).$$
Using the condition (\ref{l.la.condition.p}) we get with confidence $1-\eta/2$,
\begin{equation}\label{LK.I.up}
||(L_K+\la I)^{-1/2}(L_K-S_{\xx}^*S_{\xx})(L_K+\la I)^{-1/2}||_{\mathcal{L}(\HH)} \leq \frac{1}{2}.
\end{equation}
Therefore eqn. (\ref{1.x^n}) with (\ref{LK.I.up}) implies that with probability $1-\eta/2$,
\begin{equation}\label{LK1.2}
I_4=||L_K^{1/2}(S_{\xx}^*S_{\xx}+\la I)^{-1}(L_K+\la I)^{1/2}||_{\mathcal{L}(\HH)} \leq 2||L_K^{1/2}(L_K+\la I)^{-1/2}||_{\mathcal{L}(\HH)} \leq 2.
\end{equation}
Using (\ref{LK.I.app}), (\ref{I3}), (\ref{LK1.2}) in (\ref{p.error}), we obtain with probability $1-\eta$,
\begin{equation}\label{fzl.fl.p}
||f_{\zz,\la}-f_{\la}||_\rho \leq 4\left(\frac{\kappa M}{m\sqrt{\la}}+\sqrt{\frac{\Sigma^2\mathcal{N}(\la)}{m}}\right)\log\left(\frac{4}{\eta}\right)+\sqrt{\la}||f_\la-f_\HH||_\HH.
\end{equation}
Using the approximation error bounds from Proposition \ref{approx.err}, we get the required estimates.
\end{proof}

We derive the convergence rates of Tikhonov regularizer based on data-driven strategy of the parameter choice of $\la$ for the class of probability measure $\mathcal{P}_{\phi,b}$.
\begin{thm}\label{err.upper.bound.p.para}
Under the same assumptions of Theorem \ref{err.upper.bound.p} and hypothesis (\ref{poly.decay}), the convergence of the estimator $f_{\zz,\la}$ (\ref{fzl}) to the target function $f_\HH$ can be described as:
\begin{enumerate}[(i)]
\item  If $\phi(t)$ and $\sqrt{t}/\phi(t)$ are nondecreasing functions. Then under the parameter choice $\la\in(0,1],~\la=\Psi^{-1}(m^{-1/2})$ where $\Psi(t)=t^{\frac{1}{2}+\frac{1}{2b}}\phi(t)$, we have
$$Prob_\zz\left\{||f_{\zz,\la}-f_\HH||_\rho\leq C(\Psi^{-1}(m^{-1/2}))^{1/2}\phi(\Psi^{-1}(m^{-1/2}))\log\left(\frac{4}{\eta}\right)\right\}\geq 1-\eta$$
and
$$\lim\limits_{\tau\rightarrow\infty}\limsup\limits_{m\rightarrow\infty}\sup\limits_{\rho\in \PP_{\phi,b}} Prob_{\zz}\left\{||f_{\zz,\la}-f_\HH||_\rho>\tau(\Psi^{-1}(m^{-1/2}))^{1/2}\phi(\Psi^{-1}(m^{-1/2}))\right\}=0,$$
\item  If $\phi(t)$ and $t/\phi(t)$ are nondecreasing functions. Then under the parameter choice $\la\in(0,1],~\la=\Theta^{-1}(m^{-1/2})$ where $\Theta(t)=t^{\frac{1}{2b}}\phi(t)$, we have
$$Prob_\zz\left\{||f_{\zz,\la}-f_\HH||_\rho\leq C'\phi(\Theta^{-1}(m^{-1/2}))\log\left(\frac{4}{\eta}\right)\right\}\geq 1-\eta$$
and
$$\lim\limits_{\tau\rightarrow\infty}\limsup\limits_{m\rightarrow\infty}\sup\limits_{\rho\in \PP_{\phi,b}} Prob_{\zz}\left\{||f_{\zz,\la}-f_\HH||_\rho>\tau\phi(\Theta^{-1}(m^{-1/2}))\right\}=0.$$
\end{enumerate}
\end{thm}
\begin{proof}
(i) Let $\Psi(t)=t^{\frac{1}{2}+\frac{1}{2b}}\phi(t)$. Then it follows,
$$\lim\limits_{t \to 0}\frac{\Psi(t)}{\sqrt{t}}=\lim\limits_{t \to 0}\frac{t^2}{\Psi^{-1}(t)}=0.$$
Under the parameter choice $\la=\Psi^{-1}(m^{-1/2})$ we have,
$$\lim\limits_{m\to\infty}m\la=\infty.$$
Therefore for sufficiently large $m$,
$$\frac{1}{m\la}=\frac{\la^{\frac{1}{2b}}\phi(\la)}{\sqrt{m\la}}\leq \la^{\frac{1}{2b}}\phi(\la).$$
Under the fact $\la \leq 1$ from Theorem \ref{err.upper.bound.p} and eqn. (\ref{N(l).bound}) follows that with confidence $1-\eta$,
\begin{equation}\label{fzl.fl.inter}
||f_{\zz,\la}-f_\HH||_\rho\leq C(\Psi^{-1}(m^{-1/2}))^{1/2}\phi(\Psi^{-1}(m^{-1/2}))\log\left(\frac{4}{\eta}\right),
\end{equation}
where $C=2R+4\kappa M+4\sqrt{\beta b\Sigma^2/(b-1)}$.

Now defining $\tau:=C\log\left(\frac{4}{\eta}\right)$ gives
$$\eta=\eta_\tau=4e^{-\tau/C}.$$
The estimate (\ref{fzl.fl.inter}) can be reexpressed as
\begin{equation}\label{minimax.p1}
Prob_{\zz}\{||f_{\zz,\la}-f_\HH||_\rho>\tau(\Psi^{-1}(m^{-1/2}))^{1/2}\phi(\Psi^{-1}(m^{-1/2}))\} \leq \eta_{\tau}.
\end{equation}
(ii) Suppose $\Theta(t)=t^{\frac{1}{2b}}\phi(t).$ Then the condition (\ref{l.la.condition.p}) follows that
$$\sqrt{m\la}\geq \frac{8\kappa^2\log\left(4/\eta\right)}{\sqrt{\la}}\geq \frac{8\kappa^2}{\sqrt{\la}}.$$
Hence under the parameter choice $\la\in(0,1],~\la=\Theta^{-1}(m^{-1/2})$ we have
$$\frac{1}{m\sqrt{\la}}\leq\frac{\sqrt{\la}}{8\kappa^2\sqrt{m}} \leq\frac{\la^{\frac{1}{2}+\frac{1}{2b}}\phi(\la)}{8\kappa^2}\leq\frac{\phi(\la)}{8\kappa^2}.$$
From Theorem \ref{err.upper.bound.p} and eqn. (\ref{N(l).bound}), it follows that with confidence $1-\eta$,
\begin{equation}\label{fzl.fp.Theta.p}
||f_{\zz,\la}-f_\HH||_\rho \leq C'\phi(\Theta^{-1}(m^{-1/2}))\log\left(\frac{4}{\eta}\right),
\end{equation}
where $C':=R(\kappa+1)+M/2\kappa+4\sqrt{\beta b\Sigma^2/(b-1)}.$

Now defining $\tau:=C'\log\left(\frac{4}{\eta}\right)$ gives
$$\eta=\eta_{\tau}=4e^{-\tau/C'}.$$
The estimate (\ref{fzl.fp.Theta.p}) can be reexpressed as
\begin{equation}\label{minimax.p2}
Prob_{\zz}\left\{||f_{\zz,\la}-f_\HH||_\rho>\tau\phi(\Theta^{-1}(m^{-1/2}))\right\}\leq \eta_{\tau}.
\end{equation}
Then from eqn. (\ref{minimax.p1}) and (\ref{minimax.p2}) our conclusions follow.
\end{proof}

\begin{thm}\label{err.upper.bound.k.para}
Under the same assumptions of Theorem \ref{err.upper.bound.k} and hypothesis (\ref{poly.decay}) with the parameter choice $\la\in(0,1],~\la=\Psi^{-1}(m^{-1/2})$ where $\Psi(t)=t^{\frac{1}{2}+\frac{1}{2b}}\phi(t)$, the convergence of the estimator $f_{\zz,\la}$ (\ref{fzl}) to the target function $f_\HH$ can be described as
$$Prob_\zz\left\{||f_{\zz,\la}-f_\HH||_\HH \leq C\phi(\Psi^{-1}(m^{-1/2}))\log\left(\frac{4}{\eta}\right)\right\}\geq 1-\eta$$
and
$$\lim\limits_{\tau\rightarrow\infty}\limsup\limits_{m\rightarrow\infty}\sup\limits_{\rho\in \PP_{\phi,b}} Prob_{\zz}\left\{||f_{\zz,\la}-f_\HH||_{\HH}>\tau\phi(\Psi^{-1}(m^{-1/2}))\right\}=0.$$
\end{thm}

We obtain the following corollary as a consequence of Theorem \ref{err.upper.bound.p.para}, \ref{err.upper.bound.k.para}.
\begin{corollary}\label{err.upper.bound.cor}
Under the same assumptions of Theorem \ref{err.upper.bound.p.para}, \ref{err.upper.bound.k.para} for Tikhonov regularization with H\"{o}lder's source condition $f_\HH\in\Omega_{\phi,R},~\phi(t)=t^r$, for all $0<\eta<1$, with confidence $1-\eta$, for the parameter choice $\la=m^{-\frac{b}{2br+b+1}}$, we have
$$||f_{\zz,\la}-f_\HH||_{\HH}\leq Cm^{-\frac{br}{2br+b+1}}\log\left(\frac{4}{\eta}\right) \text{   for }0\leq r\leq 1,$$
$$||f_{\zz,\la}-f_\HH||_\rho\leq  Cm^{-\frac{2br+b}{4br+2b+2}}\log\left(\frac{4}{\eta}\right)\text{   for }0\leq r\leq \frac{1}{2}$$
and for the parameter choice $\la=m^{-\frac{b}{2br+1}}$, we have
$$||f_{\zz,\la}-f_\HH||_\rho\leq C'm^{-\frac{br}{2br+1}}\log\left(\frac{4}{\eta}\right)\text{   for }0\leq r\leq 1.$$
\end{corollary}

\subsection{Upper rates for general regularization schemes}
Bauer et al. \cite{Bauer} discussed the error estimates for general regularization schemes under general source condition. Here we study the convergence issues for general regularization schemes under general source condition and the polynomial decay of the eigenvalues of the integral operator $L_K$. We define the regularization in learning theory framework similar to considered for ill-posed inverse problems (See Section 3.1 \cite{Bauer}).

\begin{defi}\label{regularization}
A family of functions $g_\la:[0,\kappa^2]\to\RR$, $0<\la\leq \kappa^2$, is said to be the regularization if it satisfies the following conditions:
\begin{itemize}
\item $\exists D:\sup\limits_{\sigma\in(0,\kappa^2]}|\sigma g_\la(\sigma)|\leq D$.
\item $\exists B:\sup\limits_{\sigma\in(0,\kappa^2]}|g_\la(\sigma)|\leq\frac{B}{\la}$.
\item $\exists \gamma:\sup\limits_{\sigma\in(0,\kappa^2]}|1- g_\la(\sigma)\sigma|\leq \gamma$.
\item The maximal $p$ satisfying the condition:
$$\sup\limits_{\sigma\in(0,\kappa^2]}|1-g_\la(\sigma)\sigma|\sigma^p\leq\gamma_p\la^p$$
is called the qualification of the regularization $g_\la$, where $\gamma_p$ does not depend on $\la$.
\end{itemize}
\end{defi}
The properties of general regularization are satisfied by the large class of learning algorithms which are essentially all the linear regularization schemes. We refer to Section 2.2 \cite{Lu.book} for brief discussion of the regularization schemes. Here we consider general regularized solution corresponding to the above regularization:
\begin{equation}\label{gen.reg.solu}
f_{\zz,\la}=g_\la(S_{\xx}^*S_{\xx})S_{\xx}^*\yy.
\end{equation}
Here we are discussing the connection between the qualification of the regularization and general source condition \cite{Mathe}.
\begin{defi}
The qualification $p$ covers the index function $\phi$ if the function $t\to\frac{t^p}{\phi(t)}$ on $t\in(0,\kappa^2]$ is nondecreasing.
\end{defi}
The following result is a restatement of Proposition 3 \cite{Mathe}.
\begin{prop}
Suppose $\phi$ is a nondecreasing index function and the qualification of the regularization $g_\la$ covers $\phi$. Then
$$\sup\limits_{\sigma\in(0,\kappa^2]}|1-g_\la(\sigma)\sigma|\phi(\sigma)\leq c_g\phi(\la),~~c_g=\max(\gamma,\gamma_p).$$
\end{prop}
Generally, the index function $\phi$ is not stable with respect to perturbation in the integral operator $L_K$. In practice, we are only accessible to the perturbed empirical operator $S_{\xx}^*S_{\xx}$ but the source condition can be expressed in terms of $L_K$ only. So we want to control the difference $\phi(L_K)-\phi(S_{\xx}^*S_{\xx})$. In order to obtain the error estimates for general regularization, we further restrict the index functions to operator monotone functions which is defined as
\begin{defi}
A function $\phi_1:[0,d]\to[0,\infty)$ is said to be operator monotone index function if $\phi_1(0)=0$ and for every non-negative pair of self-adjoint operators $A,B$ such that $||A||,||B||\leq d$ and $A\leq B$ we have $\phi_1(A)\leq \phi_1(B)$.
\end{defi}
We consider the class of operator monotone index functions:
$$\mathcal{F}_\mu=\{\phi_1:[0,\kappa^2]\to[0,\infty)~\text{operator monotone}, \phi_1(0)=0,\phi_1(\kappa^2)\leq \mu\}.$$
For the above class of operator monotone functions from Theorem 1 \cite{Bauer}, given $\phi_1\in F_\mu$ there exists $c_{\phi_1}$ such that
$$||\phi_1(S_{\xx}^*S_{\xx})-\phi_1(L_K)||_{\mathcal{L}(\HH)}\leq c_{\phi_1}\phi_1(||S_{\xx}^*S_{\xx}-L_K||_{\mathcal{L}(\HH)}).$$
Here we observe that the rate of convergence of $\phi_1(S_{\xx}^*S_{\xx})$ to $\phi_1(L_K)$ is slower than the convergence rate of $S_{\xx}^*S_{\xx}$ to $L_K$. Therefore we consider the following class of index functions:
$$\mathcal{F}=\{\phi=\phi_2\phi_1:\phi_1\in\mathcal{F}_\mu,\phi_2:[0,\kappa^2]\to[0,\infty) \text{ nondecreasing Lipschitz},\phi_2(0)=0\}.$$
The splitting of $\phi=\phi_2\phi_1$ is not unique. So we can take $\phi_2$ as a Lipschitz function with Lipschitz constant $1$. Now using Corollary 1.2.2 \cite{Peller} we get
\vspace{-.2 cm}
$$||\phi_2(S_{\xx}^*S_{\xx})-\phi_2(L_K)||_{\mathcal{L}_2(\HH)}\leq ||S_{\xx}^*S_{\xx}-L_K||_{\mathcal{L}_2(\HH)}.$$

General source condition $f_\HH\in\Omega_{\phi,R}$ corresponding to index class functions $\mathcal{F}$ covers wide range of source conditions as H\"{o}lder's source condition $\phi(t)=t^r$, logarithm source condition $\phi(t)=t^p\log^{-\nu}\left(\frac{1}{t}\right)$. Following the analysis of Bauer et al. \cite{Bauer} we develop the error estimates of general regularization for the index class function $\mathcal{F}$ under the suitable priors on the probability measure $\rho$.
\begin{thm}\label{err.upper.bound.gen.k}
Let $\zz$ be i.i.d. samples drawn according to the probability measure $\rho\in \PP_\phi$. Suppose $f_{\zz,\la}$ is the regularized solution (\ref{gen.reg.solu}) corresponding to general regularization and the qualification of the regularization covers $\phi$. Then for all $0<\eta<1$, with confidence $1-\eta$, the following upper bound holds:
$$||f_{\zz,\la}-f_\HH||_{\HH} \leq \left\{Rc_g(1+c_{\phi_1})\phi(\la)+\frac{4R\mu\gamma\kappa^2}{\sqrt{m}}+\frac{2\sqrt{2}\nu_1\kappa M}{m\la}+\sqrt{\frac{8\nu_1^2\Sigma^2\mathcal{N}(\la)}{m\la}}\right\}\log\left(\frac{4}{\eta}\right)$$
provided that
\begin{equation}\label{l.la.condition.gen.k}
\sqrt{m}\la \geq 8\kappa^2\log(4/\eta).
\end{equation}
\end{thm}
\begin{proof}
We consider the error expression for general regularized solution (\ref{gen.reg.solu}),
\begin{equation}\label{fzl.fp.gen}
f_{\zz,\la}-f_\HH=g_\la(S_\xx^*S_\xx)(S_{\xx}^*\yy-S_{\xx}^*S_{\xx}f_\HH)-r_\la(S_{\xx}^*S_{\xx})f_\HH,
\end{equation}
where $r_\la(\sigma)=1-g_\la(\sigma)\sigma$.

Now the first term can be expressed as
\begin{equation*}
g_\la(S_{\xx}^*S_{\xx})(S_{\xx}^*\yy-S_{\xx}^*S_{\xx}f_\HH)=g_\la(S_{\xx}^*S_{\xx})(S_{\xx}^*S_{\xx}+\la I)^{1/2}(S_{\xx}^*S_{\xx}+\la I)^{-1/2}(L_K+\la I)^{1/2} (L_K+\la I)^{-1/2}(S_{\xx}^*\yy-S_{\xx}^*S_{\xx}f_\HH).
\end{equation*}
On applying RKHS-norm we get,
\begin{equation}\label{gl.first.H}
||g_\la(S_{\xx}^*S_{\xx})(S_{\xx}^*\yy-S_{\xx}^*S_{\xx}f_\HH)||_\HH \leq I_2I_5||g_\la(S_{\xx}^*S_{\xx})(S_{\xx}^*S_{\xx}+\la I)^{1/2}||_{\mathcal{L}(\HH)},
\end{equation}
where $I_2=||(L_K+\la I)^{-1/2}(S_{\xx}^*\yy-S_{\xx}^*S_{\xx}f_\HH)||_{\HH}$ and $I_5=||(S_{\xx}^*S_{\xx}+\la I)^{-1/2}(L_K+\la I)^{1/2}||_{\mathcal{L}(\HH)}$.

The estimate of $I_2$ can be obtained from the first estimate of Proposition \ref{main.bound} and from the second estimate of Proposition \ref{main.bound} with the condition (\ref{l.la.condition.gen.k}) we obtain with probability $1-\eta/2$,
$$||(L_K+\la I)^{-1/2}(L_K-S_{\xx}^*S_{\xx})(L_K+\la I)^{-1/2}||_{\mathcal{L}(\HH)} \leq \frac{1}{\la}||S_{\xx}^*S_{\xx}-L_K||_{\mathcal{L}(\HH)}\leq \frac{4\kappa^2}{\sqrt{m}\la}\log\left(\frac{4}{\eta}\right)\leq \frac{1}{2}.$$
which implies that with confidence $1-\eta/2$,
\begin{eqnarray}\label{I.LK.lI}
I_5&=&||(S_{\xx}^*S_{\xx}+\la I)^{-1/2}(L_K+\la I)^{1/2}||_{\mathcal{L}(\HH)}=||(L_K+\la I)^{1/2}(S_{\xx}^*S_{\xx}+\la I)^{-1}(L_K+\la I)^{1/2}||_{\mathcal{L}(\HH)}^{1/2} \\ \nonumber
&=&||\{I-(L_K+\la I)^{-1/2}(L_K-S_{\xx}^*S_{\xx})(L_K+\la I)^{-1/2}\}^{-1}||_{\mathcal{L}(\HH)}^{1/2} \leq \sqrt{2}.
\end{eqnarray}
From the properties of the regularization we have,
\begin{equation}\label{g.s.r.s}
||g_\la(S_{\xx}^*S_{\xx})(S_{\xx}^*S_{\xx})^{1/2}||_{\mathcal{L}(\HH)} \leq \sup\limits_{0<\sigma\leq \kappa^2}|g_\la
(\sigma)\sqrt{\sigma}|=\left(\sup_{0<\sigma\leq \kappa^2}|g_\la(\sigma)\sigma|\sup_{0<\sigma\leq \kappa^2}|g_\la(\sigma)|\right)^{1/2} \leq \sqrt{\frac{BD}{\la}}.
\end{equation}
Hence it follows,
\begin{equation}\label{g.s.r.s.1}
||g_\la(S_{\xx}^*S_{\xx})(S_{\xx}^*S_{\xx}+\la I)^{1/2}||_{\mathcal{L}(\HH)} \leq \sup\limits_{0<\sigma\leq \kappa^2}|g_\la
(\sigma)(\sigma+\la)^{1/2}| \leq \sup\limits_{0<\sigma\leq \kappa^2}|g_\la(\sigma)\sqrt{\sigma}|+\sqrt{\la}\sup\limits_{0<\sigma\leq\kappa^2}|g_\la
(\sigma)|\leq\frac{\nu_1}{\sqrt{\la}},
\end{equation}
where $\nu_1:=B+\sqrt{BD}$.

Therefore using (\ref{LK.I.app}), (\ref{I.LK.lI}) and (\ref{g.s.r.s.1}) in eqn. (\ref{gl.first.H}) we conclude that with probability $1-\eta$,
\begin{equation}\label{gl.first}
||g_\la(S_{\xx}^*S_{\xx})(S_{\xx}^*\yy-S_{\xx}^*S_{\xx}f_\HH)||_\HH \leq 2\sqrt{2}\nu_1\left\{\frac{\kappa M}{m\la}+\sqrt{\frac{\Sigma^2\mathcal{N}(\la)}{m\la}}\right\}\log\left(\frac{4}{\eta}\right).
\end{equation}
Now we consider the second term,
\begin{eqnarray*}
r_\la(S_{\xx}^*S_{\xx})f_\HH&=& r_\la(S_{\xx}^*S_{\xx})\phi(L_K)v=r_\la(S_{\xx}^*S_{\xx})\phi(S_{\xx}^*S_{\xx})v +r_\la(S_{\xx}^*S_{\xx}) \phi_2(S_{\xx}^*S_{\xx})(\phi_1(L_K)-\phi_1(S_{\xx}^*S_{\xx}))v  \\
&&+r_\la(S_{\xx}^*S_{\xx})(\phi_2(L_K)-\phi_2(S_{\xx}^*S_{\xx}))\phi_1(L_K)v.
\end{eqnarray*}
Employing RKHS-norm we get
$$||r_\la(S_{\xx}^*S_{\xx})f_\HH||_\HH \leq Rc_g\phi(\la)+Rc_gc_{\phi_1}\phi_2(\la) \phi_1(||L_K-S_{\xx}^*S_{\xx}||_{\mathcal{L}(\HH)})+R\mu\gamma||L_K-S_{\xx}^*S_{\xx}||_{\mathcal{L}_2(\HH)}.$$
Here we used the fact that if the qualification of the regularization covers $\phi=\phi_1\phi_2$, then the qualification also covers $\phi_1$ and $\phi_2$ both separately.

From eqn. (\ref{Sx.Sx.LK}) and (\ref{l.la.condition.gen.k}) we have with probability $1-\eta/2$,
\begin{equation}\label{Sx.Sx.LK.l.2}
||S_{\xx}^*S_{\xx}-L_K||_{\mathcal{L}(\HH)} \leq\frac{4\kappa^2}{\sqrt{m}}\log\left(\frac{4}{\eta}\right)\leq \la/2.
\end{equation}
Therefore with probability $1-\eta/2$,
\begin{equation}\label{rl.second}
||r_\la(S_{\xx}^*S_{\xx})f_\HH||_{\HH} \leq Rc_g(1+c_{\phi_1})\phi(\la)+\frac{4R\mu\gamma\kappa^2}{\sqrt{m}}\log\left(\frac{4}{\eta}\right).
\end{equation}
Combining the bounds (\ref{gl.first}) and (\ref{rl.second}) we get the desired result.
\end{proof}

\begin{thm}\label{err.upper.bound.gen.p}
Let $\zz$ be i.i.d. samples drawn according to the probability measure $\rho\in \PP_\phi$ and $f_{\zz,\la}$ is the regularized solution (\ref{gen.reg.solu}) corresponding to general regularization. Then for all $0<\eta<1$, with confidence $1-\eta$, the following upper bounds holds:
\begin{enumerate}[(i)]
\item  If the qualification of the regularization covers $\phi$,
$$||f_{\zz,\la}-f_\HH||_\rho \leq \left\{Rc_g(1+c_{\phi_1})(\kappa+\sqrt{\la})\phi(\la)+\frac{4R\mu\gamma\kappa^2(\kappa+\sqrt{\la})}{\sqrt{m}} +\frac{2\sqrt{2}\nu_2\kappa M}{m\sqrt{\la}}+\sqrt{\frac{8\nu_2^2\Sigma^2\mathcal{N}(\la)}{m}}\right\}\log\left(\frac{4}{\eta}\right),$$
\item  If the qualification of the regularization covers $\phi(t)\sqrt{t}$,
$$||f_{\zz,\la}-f_\HH||_\rho \leq \left\{2Rc_g(1+c_{\phi_1})\phi(\la)\sqrt{\la}+\frac{4R\mu(\gamma+c_g)\kappa^2\sqrt{\la}}{\sqrt{m}}+\frac{2\sqrt{2}\nu_2\kappa M}{m\sqrt{\la}}+\sqrt{\frac{8\nu_2^2\Sigma^2\mathcal{N}(\la)}{m}}\right\}\log\left(\frac{4}{\eta}\right)$$
\end{enumerate}
provided that
\begin{equation}\label{l.la.condition.gen.p}
\sqrt{m}\la \geq 8\kappa^2\log(4/\eta).
\end{equation}
\end{thm}
\begin{proof}
Here we establish $\LL^2$-norm estimate for the error expression:
\begin{equation*}
f_{\zz,\la}-f_\HH=g_\la(S_\xx^*S_\xx)(S_{\xx}^*\yy-S_{\xx}^*S_{\xx}f_\HH)-r_\la(S_{\xx}^*S_{\xx})f_\HH.
\end{equation*}
On applying $\LL^2$-norm in the first term we get,
\begin{equation}\label{gl.initial}
||g_\la(S_{\xx}^*S_{\xx})(S_{\xx}^*\yy-S_{\xx}^*S_{\xx}f_\HH)||_\rho \leq I_2I_5||L_K^{1/2}g_\la(S_{\xx}^*S_{\xx})(S_{\xx}^*S_{\xx}+\la I)^{1/2}||_{\mathcal{L}(\HH)},
\end{equation}
where $I_2=||(L_K+\la I)^{-1/2}(S_{\xx}^*\yy-S_{\xx}^*S_{\xx}f_\HH)||_\HH$ and $I_5=||(S_{\xx}^*S_{\xx}+\la I)^{-1/2}(L_K+\la I)^{1/2}||_{\mathcal{L}(\HH)}$.

The estimates of $I_2$ and $I_5$ can be obtained from Proposition \ref{main.bound} and Theorem \ref{err.upper.bound.gen.k} respectively. Now we consider
\begin{eqnarray*}
||L_K^{1/2}g_\la(S_{\xx}^*S_{\xx})(S_{\xx}^*S_{\xx}+\la I)^{1/2}||_{\mathcal{L}(\HH)} &\leq& ||L_K^{1/2}-(S_{\xx}^*S_{\xx})^{1/2}||_{\mathcal{L}(\HH)}||g_\la(S_{\xx}^*S_{\xx})(S_{\xx}^*S_{\xx}+\la I)^{1/2}||_{\mathcal{L}(\HH)} \\
&&+||(S_{\xx}^*S_{\xx})^{1/2}g_\la(S_{\xx}^*S_{\xx})(S_{\xx}^*S_{\xx}+\la I)^{1/2}||_{\mathcal{L}(\HH)}.
\end{eqnarray*}
Since $\phi(t)=\sqrt{t}$ is operator monotone function. Therefore from eqn. (\ref{Sx.Sx.LK.l.2}) with probability $1-\eta/2$, we get
$$||L_K^{1/2}-(S_{\xx}^*S_{\xx})^{1/2}||_{\mathcal{L}(\HH)} \leq (||L_K-S_{\xx}^*S_{\xx}||_{\mathcal{L}(\HH)})^{1/2} \leq \sqrt{\la}.$$
Then using the properties of the regularization and eqn. (\ref{g.s.r.s}) we conclude that with probability $1-\eta/2$,
\begin{eqnarray}\label{Lk.gl}
||L_K^{1/2}g_\la(S_{\xx}^*S_{\xx})(S_{\xx}^*S_{\xx}+\la I)^{1/2}||_{\mathcal{L}(\HH)} &\leq& \sqrt{\la}\sup\limits_{0<\sigma\leq \kappa^2}|g_\la
(\sigma)(\sigma+\la)^{1/2}|+\sup\limits_{0<\sigma\leq \kappa^2}|g_\la(\sigma)(\sigma^2+\la\sigma)^{1/2}|             \\ \nonumber
&\leq& \sup\limits_{0<\sigma\leq \kappa^2}|g_\la(\sigma)\sigma|+\la\sup\limits_{0<\sigma\leq\kappa^2}|g_\la(\sigma)| +2\sqrt{\la}\sup\limits_{0<\sigma\leq \kappa^2}|g_\la(\sigma)\sqrt{\sigma}|                                           \\ \nonumber
&\leq& B+D+2\sqrt{BD}=\nu_2 \text{(let)}.
\end{eqnarray}
From eqn. (\ref{gl.initial}) with (\ref{LK.I.app}), (\ref{I.LK.lI}) and (\ref{Lk.gl}) we obtain with probability $1-\eta$,
\begin{equation}\label{gl.first.k}
||g_\la(S_{\xx}^*S_{\xx})(S_{\xx}^*\yy-S_{\xx}^*S_{\xx}f_\HH)||_\rho \leq 2\sqrt{2}\nu_2\left\{\frac{\kappa M}{m\sqrt{\la}}+\sqrt{\frac{\Sigma^2\mathcal{N}(\la)}{m}}\right\}\log\left(\frac{4}{\eta}\right).
\end{equation}
The second term can be expressed as
\begin{eqnarray*}
||r_\la(S_{\xx}^*S_{\xx})f_\HH||_\rho &\leq& ||L_K^{1/2}-(S_{\xx}^*S_{\xx})^{1/2}||_{\mathcal{L}(\HH)}||r_\la(S_{\xx}^*S_{\xx})f_\HH||_\HH +||(S_{\xx}^*S_{\xx})^{1/2}r_\la(S_{\xx}^*S_{\xx})f_\HH||_\HH\\ \nonumber
&\leq& ||L_K-S_{\xx}^*S_{\xx}||_{\mathcal{L}(\HH)}^{1/2}||r_\la(S_{\xx}^*S_{\xx})f_\HH||_\HH
+||r_{\la}(S_{\xx}^*S_{\xx})(S_{\xx}^*S_{\xx})^{1/2}\phi(S_{\xx}^*S_{\xx})v||_\HH \\     \nonumber
&&+||r_\la(S_{\xx}^*S_{\xx})(S_{\xx}^*S_{\xx})^{1/2}\phi_2(S_{\xx}^*S_{\xx}) (\phi_1(S_{\xx}^*S_{\xx})-\phi_1(L_K))v||_{\HH} \\ \nonumber
&&+||r_\la(S_{\xx}^*S_{\xx})(S_{\xx}^*S_{\xx})^{1/2}(\phi_2(S_{\xx}^*S_{\xx})-\phi_2(L_K)) \phi_1(L_K)v||_{\HH}.
\end{eqnarray*}
Here two cases arises:\\ \\
{\bf Case 1.} If the qualification of the regularization covers $\phi$. Then we get with confidence $1-\eta/2$,
$$||r_\la(S_{\xx}^*S_{\xx})f_\HH||_\rho \leq (\kappa+\sqrt{\la})\left(R c_g(1+c_{\phi_1})\phi(\la)+R\mu\gamma||S_{\xx}^*S_{\xx}-L_K||_{\mathcal{L}_2(\HH)}\right).$$
Therefore using eqn. (\ref{Sx.Sx.LK}) we obtain with probability $1-\eta/2$,
\begin{equation}\label{rl.second.k}
||r_\la(S_{\xx}^*S_{\xx})f_\HH||_\rho \leq (\kappa+\sqrt{\la})\left(R c_g(1+c_{\phi_1})\phi(\la)+\frac{4R\mu\gamma\kappa^2}{\sqrt{m}}\log\left(\frac{4}{\eta}\right)\right).
\end{equation}
{\bf Case 2.} If the qualification of the regularization covers $\phi(t)\sqrt{t}$, we get with probability $1-\eta/2$,
\begin{equation}\label{rl.second.k1}
||r_\la(S_{\xx}^*S_{\xx})f_\HH||_\rho \leq 2R c_g(1+c_{\phi_1})\phi(\la)\sqrt{\la}+4R\mu(\gamma+c_g)\kappa^2\sqrt{\frac{\la}{m}}\log\left(\frac{4}{\eta}\right).
\end{equation}

Combining the error estimates (\ref{gl.first.k}), (\ref{rl.second.k}) and (\ref{rl.second.k1}) we get the desired results.
\end{proof}

We discuss the convergence rates of general regularizer based on data-driven strategy of the parameter choice of $\la$ for the class of probability measure $\mathcal{P}_{\phi,b}$.
\begin{thm}\label{err.upper.bound.gen.k.para}
Under the same assumptions of Theorem \ref{err.upper.bound.gen.k} and hypothesis (\ref{poly.decay}) with the parameter choice $\la\in(0,1],~\la=\Psi^{-1}(m^{-1/2})$ where $\Psi(t)=t^{\frac{1}{2}+\frac{1}{2b}}\phi(t)$, the convergence of the estimator $f_{\zz,\la}$ (\ref{gen.reg.solu}) to the target function $f_\HH$ can be described as
$$Prob_\zz\left\{||f_{\zz,\la}-f_\HH||_{\HH}\leq \widetilde{C}\phi(\Psi^{-1}(m^{-1/2}))\log\left(\frac{4}{\eta}\right)\right\}\geq 1-\eta,$$
where $\widetilde{C}=Rc_g(1+c_{\phi_1})+4R\mu\gamma\kappa^2+2\sqrt{2}\nu_1\kappa M+\sqrt{8\beta b\nu_1^2\Sigma^2/(b-1)}$ and
$$\lim\limits_{\tau\rightarrow\infty}\limsup\limits_{m\rightarrow\infty}\sup\limits_{\rho\in \PP_{\phi,b}} Prob_{\zz}\left\{||f_{\zz,\la}-f_\HH||_{\HH}>\tau\phi(\Psi^{-1}(m^{-1/2}))\right\}=0.$$
\end{thm}

\begin{thm}\label{err.upper.bound.gen.p.para.choice}
Under the same assumptions of Theorem \ref{err.upper.bound.gen.p} and hypothesis (\ref{poly.decay}), the convergence of the estimator $f_{\zz,\la}$ (\ref{gen.reg.solu}) to the target function $f_\HH$ can be described as
\begin{enumerate}[(i)]
\item  If the qualification of the regularization covers $\phi$. Then under the parameter choice $\la\in(0,1],~\la=\Theta^{-1}(m^{-1/2})$ where $\Theta(t)=t^{\frac{1}{2b}}\phi(t)$, we have
$$Prob_\zz\left\{||f_{\zz,\la}-f_\HH||_\rho\leq \widetilde{C}_1\phi(\Theta^{-1}(m^{-1/2}))\log\left(\frac{4}{\eta}\right)\right\}\geq 1-\eta,$$
where $\widetilde{C}_1=Rc_g(1+c_{\phi_1})(\kappa+1)+4R\mu\gamma\kappa^2(\kappa+1)+\nu_2 M/2\sqrt{2}\kappa+\sqrt{8\beta b\nu_2^2\Sigma^2/(b-1)}$ and
$$\lim\limits_{\tau\rightarrow\infty}\limsup\limits_{m\rightarrow\infty}\sup\limits_{\rho\in \PP_{\phi,b}} Prob_{\zz}\left\{||f_{\zz,\la}-f_\HH||_\rho>\tau\phi(\Theta^{-1}(m^{-1/2}))\right\}=0,$$
\item  If the qualification of the regularization covers $\phi(t)\sqrt{t}$. Then under the parameter choice $\la\in(0,1],~\la=\Psi^{-1}(m^{-1/2})$ where $\Psi(t)=t^{\frac{1}{2}+\frac{1}{2b}}\phi(t)$, we have
$$Prob_\zz\left\{||f_{\zz,\la}-f_\HH||_\rho\leq \widetilde{C}_2(\Psi^{-1}(m^{-1/2}))^{1/2}\phi(\Psi^{-1}(m^{-1/2}))\log\left(\frac{4}{\eta}\right)\right\}\geq 1-\eta,$$
where $\widetilde{C}_2=2Rc_g(1+c_{\phi_1})+4R\mu(\gamma+c_g)\kappa^2+2\sqrt{2}\nu_2\kappa M+\sqrt{8\beta b\nu_2^2\Sigma^2/(b-1)}$ and
$$\lim\limits_{\tau\rightarrow\infty}\limsup\limits_{m\rightarrow\infty}\sup\limits_{\rho\in \PP_{\phi,b}} Prob_{\zz}\left\{||f_{\zz,\la}-f_\HH||_\rho>\tau(\Psi^{-1}(m^{-1/2}))^{1/2}\phi(\Psi^{-1}(m^{-1/2}))\right\}=0.$$
\end{enumerate}
\end{thm}

We obtain the following corollary as a consequence of Theorem \ref{err.upper.bound.gen.k.para}, \ref{err.upper.bound.gen.p.para.choice}.
\begin{corollary}\label{err.upper.bound.gen.cor}
Under the same assumptions of Theorem \ref{err.upper.bound.gen.k.para}, \ref{err.upper.bound.gen.p.para.choice} for Tikhonov regularization with H\"{o}lder's source condition $f_\HH\in\Omega_{\phi,R},~\phi(t)=t^r$, for all $0<\eta<1$, with confidence $1-\eta$, for the parameter choice $\la=m^{-\frac{b}{2br+b+1}}$, we have
$$||f_{\zz,\la}-f_\HH||_{\HH}\leq \widetilde{C}m^{-\frac{br}{2br+b+1}}\log\left(\frac{4}{\eta}\right) \text{   for }0\leq r\leq 1,$$
$$||f_{\zz,\la}-f_\HH||_\rho\leq  \widetilde{C}_2m^{-\frac{2br+b}{4br+2b+2}}\log\left(\frac{4}{\eta}\right)\text{   for }0\leq r\leq \frac{1}{2}$$
and for the parameter choice $\la=m^{-\frac{b}{2br+1}}$, we have
$$||f_{\zz,\la}-f_\HH||_\rho\leq \widetilde{C}_1m^{-\frac{br}{2br+1}}\log\left(\frac{4}{\eta}\right)\text{   for }0\leq r\leq 1.$$
\end{corollary}

\begin{rem}
It is important to observe from Corollary \ref{err.upper.bound.cor}, \ref{err.upper.bound.gen.cor} that using the concept of operator monotonicity of index function we are able to achieve the same error estimates for general regularization as of Tikhonov regularization up to a constant multiple.
\end{rem}
\begin{rem}(Related work)
Corollary \ref{err.upper.bound.cor} provides the order of convergence same as of Theorem 1 \cite{Caponnetto} for Tikhonov regularization under the H\"{o}lder's source condition $f_\HH\in\Omega_{\phi,R}$ for $\phi(t)=t^r~\left(\frac{1}{2}\leq r\leq 1\right)$ and the polynomial decay of the eigenvalues (\ref{poly.decay}). Under the fact $\mathcal{N}(\la)\leq\frac{\kappa^2}{\la}$ from Theorem \ref{err.upper.bound.gen.k}, \ref{err.upper.bound.gen.p} we obtain the similar estimates as of Theorem 10 \cite{Bauer} for general regularization schemes without the polynomial decay condition of the eigenvalues (\ref{poly.decay}).
\end{rem}
\begin{rem}
For the real valued functions and multi-task algorithms (the output space $Y\subset\RR^m$ for some $m\in\NN$) we can obtain the error estimates from our analysis without imposing any condition on the conditional probability measure (\ref{Y.leq.M.2}) for the bounded output space $Y$.
\end{rem}
\begin{rem}
We can address the convergence issues of binary classification problem \cite{Boucheron} using our error estimates as similar to discussed in Section 3.3 \cite{Bauer} and Section 5 \cite{Smale3}.
\end{rem}
\subsection{Lower rates for general learning algorithms}
In this section, we discuss the estimates of minimum possible error over a subclass of the probability measures $\mathcal{P}_{\phi,b}$ parameterized by suitable functions $f\in\HH$. Throughout this section we assume that $Y$ is finite-dimensional.

Let $\{v_j\}_{j=1}^d$ be a basis of $Y$ and $f \in \Omega_{\phi,R}$. Then we parameterize the probability measure based on the function $f$,
\begin{equation}\label{p.f}
\rho_f(x,y):=\frac{1}{2dL}\sum\limits_{j=1}^d\left(a_j(x)\delta_{y+dLv_j}+b_j(x)\delta_{y-dLv_j}\right)\nu(x),
\end{equation}
where $a_j(x)=L-\langle f,K_xv_j\rangle_{\HH}$, $b_j(x)=L+\langle f,K_xv_j\rangle_{\HH}$, $L=4\kappa\phi(\kappa^2)R$ and $\delta_\xi$ denotes the Dirac measure with unit mass at $\xi$. It is easy to observe that the marginal distribution of $\rho_f$ over $X$ is $\nu$ and the regression function for the probability measure $\rho_f$ is $f$ (see Proposition 4 \cite{Caponnetto}). In addition to this, for the conditional probability measure $\rho_f(y|x)$ we have,
\begin{equation*}
\int_Y\left(e^{||y-f(x)||_Y/M}-\frac{||y-f(x)||_Y}{M}-1\right)d\rho_f(y|x)\leq\left(d^2L^2-||f(x)||_Y^2\right)\sum\limits_{i=2}^\infty \frac{(dL+||f(x)||_Y)^{i-2}}{M^ii!}\leq\frac{\Sigma^2}{2M^2}
\end{equation*}
provided that
$$dL+L/4\leq M \text{ and }\sqrt{2}dL\leq\Sigma.$$
We assume that the eigenvalues of the integral operator $L_K$ follow the polynomial decay (\ref{poly.decay}) for the marginal probability measure $\nu$. Then we conclude that the probability measure $\rho_f$ parameterized by $f$ belongs to the class $\mathcal{P}_{\phi,b}$.

The concept of information theory such as the Kullback-Leibler information and Fano inequalities (Lemma 3.3 \cite{DeVore}) are the main ingredients in the analysis of lower bounds. In the literature \cite{Caponnetto,DeVore}, the closeness of probability measures is described through Kullback-Leibler information: Given two probability measures $\rho_1$ and $\rho_2$, it is defined as
$$\mathcal{K}(\rho_1,\rho_2):=\int_Z\log(g(z))d\rho_1(z),$$
where $g$ is the density of $\rho_1$ with respect to $\rho_2$, that is, $\rho_1(E)=\int_Eg(z)d\rho_2(z)$ for all measurable sets $E$.

Following the analysis of Caponnetto et al. \cite{Caponnetto} and DeVore et al. \cite{DeVore} we establish the lower rates of accuracy that can be attained by any learning algorithm.
\begin{thm}\label{err.lower.bound.p}
Let $\zz$ be i.i.d. samples drawn according to the probability measure $\rho\in\mathcal{P}_{\phi,b}$ under the hypothesis $dim(Y)=d<\infty$. Then for any learning algorithm ($\zz\to f_\zz\in\HH$) there exists a probability measure $\rho_*\in \mathcal{P}_{\phi,b}$ and $f_{\rho_*}\in \HH$ such that for all $0<\varepsilon<\varepsilon_o$, $f_\zz$ can be approximated as
$$Prob_{\zz}\left\{||f_{\zz}-f_{\rho_*}||_{\LL_{\nu}^2(X)} >\varepsilon/2\right\} \geq \min\left\{\frac{1}{1+e^{-\ell_\varepsilon/24}}, \vartheta e^{\left(\frac{\ell_\varepsilon}{48}-\frac{64m\varepsilon^2}{15dL^2}\right)}\right\}$$
where $\vartheta=e^{-3/e}$, $\ell_\varepsilon=\left\lfloor\left(\frac{\al}{\psi^{-1}(\varepsilon/R)}\right)^{1/b}\right\rfloor$ and $\psi(t)=\sqrt{t}\phi(t)$.
\end{thm}
\begin{proof}
To estimate the lower rates of learning algorithms, we generate $N_\varepsilon$-functions belonging to $\Omega_{\phi,R}$ for given $\varepsilon>0$ such that (\ref{f.i.j.bound}) holds. Then we construct the probability measures $\rho_i\in\PP_{\phi,b}$ from (\ref{p.f}), parameterized by these functions $f_i$'s $(1\leq i\leq N_\varepsilon)$. On applying Lemma 3.3 \cite{DeVore}, we obtain the lower convergence rates using Kullback-Leibler information.

For given $\varepsilon>0$, we define
$$g=\sum\limits_{n=1}^\ell\frac{\varepsilon\sigma^ne_n}{\sqrt{\ell t_n}\phi(t_n)},$$
where $\sigma=(\sigma^1,\ldots,\sigma^\ell)\in \{-1,+1\}^\ell$, $t_n$'s are the eigenvalues of the integral operator $L_K$, $e_n$'s are the eigenvectors of the integral operator $L_K$ and the orthonormal basis of RKHS $\HH$. Under the decay condition on the eigenvalues $\al \leq n^bt_n$, we get
$$||g||_{\HH}^2=\sum\limits_{n=1}^\ell\frac{\varepsilon^2}{\ell t_n\phi^2(t_n)} \leq \sum\limits_{n=1}^\ell\frac{\varepsilon^2}{\ell\left(\frac{\al}{n^b}\right)\phi^2\left(\frac{\al}{n^b}\right)} \leq \frac{\varepsilon^2}{\left(\frac{\al}{\ell^b}\right)\phi^2\left(\frac{\al}{\ell^b}\right)}.$$
Hence $f=\phi(L_K)g\in\Omega_{\phi,R}$ provided that $||g||_{\HH}\leq R$ or equivalently,
\begin{equation}\label{m.bound}
\ell\leq \left(\frac{\al}{\psi^{-1}(\varepsilon/R)}\right)^{1/b},
\end{equation}
where $\psi(t)=\sqrt{t}\phi(t)$.

For $\ell=\ell_\varepsilon=\left\lfloor\left(\frac{\al}{\psi^{-1}(\varepsilon/R)}\right)^{1/b}\right\rfloor$, choose $\varepsilon_o$ such that $\ell_{\varepsilon_o}>16$. Then according to Proposition 6 \cite{Caponnetto}, for every positive $\varepsilon<\varepsilon_o~(\ell_\varepsilon>\ell_{\varepsilon_o})$ there exists $N_\varepsilon \in \NN$ and $\sigma_1,\ldots,\sigma_{N_\varepsilon} \in \{-1,+1\}^{\ell_\varepsilon}$ such that
\begin{equation}\label{sigma.i.j}
\sum_{n=1}^{\ell_\varepsilon}(\sigma_i^n-\sigma_j^n)^2 \geq \ell_{\varepsilon}, \text{    for all } 1 \leq i,j \leq N_\varepsilon,~i \neq j
\end{equation}
and
\begin{equation}\label{N.epsilon}
N_\varepsilon \geq e^{\ell_\varepsilon/24}.
\end{equation}
Now we suppose $f_i=\phi(L_K)g_i$ and for $\varepsilon>0$,
$$g_i=\sum\limits_{n=1}^{\ell_\varepsilon}\frac{\varepsilon\sigma_i^ne_n}{\sqrt{\ell_\varepsilon t_n}\phi(t_n)}, \text{ for } i=1,\ldots,N_\varepsilon,$$
where $\sigma_i=(\sigma_i^1,\ldots,\sigma_i^{\ell_\varepsilon})\in \{-1,+1\}^{\ell_\varepsilon}$. Then we have,
$$||f_i-f_j||_{\LL^2_{\nu}(X)}^2=||L_K^{1/2}\phi(L_K)(g_i-g_j)||_{\HH}^2 =\sum_{n=1}^{\ell_\varepsilon}\frac{\varepsilon^2}{\ell_{\varepsilon}}(\sigma_i^n-\sigma_j^n)^2,~~~\forall~~1\leq i,j\leq N_\varepsilon.$$
The fact $(\sigma_i^n-\sigma_j^n)^2\leq 4$ together with condition (\ref{sigma.i.j}) implies
\begin{equation}\label{f.i.j.bound}
\varepsilon \leq ||f_i-f_j||_{\LL_{\nu}^2(X)} \leq 2\varepsilon.
\end{equation}
It is easy to see that the probability measures $\rho_{f_i}$'s defined by (\ref{p.f}) belongs to the class $\PP_{\phi,b}$.

We define the sets,
$$A_i=\left\{\zz:||f_{\zz}-f_i||_{\LL_{\nu}^2(X)}<\frac{\varepsilon}{2}\right\}, \text{ for } 1\leq i \leq N_\varepsilon.$$
It is clear from (\ref{f.i.j.bound}) that $A_i$'s are disjoint sets. On applying Lemma 3.3 \cite{DeVore} with probability measures $\rho_{f_i}^m,~1 \leq i \leq N_\varepsilon$, we obtain that either
\begin{equation}\label{p.either.bound}
p:=\max\limits_{1\leq i \leq N_\varepsilon}\rho_{f_i}^m(A_i^c) \geq \frac{N_{\varepsilon}}{N_\varepsilon+1}
\end{equation}
or
\begin{equation}\label{leibler.l.bound}
\min\limits_{1 \leq j \leq N_\varepsilon} \frac{1}{N_\varepsilon}\sum\limits_{i=1, i \neq j}^{N_\varepsilon}\mathcal{K}(\rho_{f_i}^m,\rho_{f_j}^m) \geq \Psi_{N_\varepsilon}(p),
\end{equation}
where $\Psi_{N_\varepsilon}(p)=\log(N_\varepsilon)+(1-p)\log\left(\frac{1-p}{p}\right) -p\log\left(\frac{N_\varepsilon-p}{p}\right)$.
Further,
\begin{equation}\label{psi.p.bound}
\Psi_{N_\varepsilon}(p)\geq(1-p)\log(N_\varepsilon)+(1-p)\log(1-p)-\log(p)+2p\log(p) \geq -\log(p)+\log(\sqrt{N_\varepsilon})-3/e.
\end{equation}
Since minimum value of $x\log(x)$ is $-1/e$ on $[0,1]$. From Proposition 4 \cite{Caponnetto} and the eqn. (\ref{f.i.j.bound}) we have,
\begin{equation}\label{leibler.u.bound}
\mathcal{K}(\rho_{f_i}^m,\rho_{f_j}^m)= m\mathcal{K}(\rho_{f_i},\rho_{f_j}) \leq \frac{16m}{15dL^2}||f_i-f_j||_{\LL_{\nu}^2(X)}^2 \leq \frac{64m\varepsilon^2}{15dL^2}.
\end{equation}
Therefore eqn. (\ref{leibler.l.bound}), together with (\ref{psi.p.bound}) and (\ref{leibler.u.bound}) implies
\begin{equation}\label{p.or.bound}
p\geq \sqrt{N_\varepsilon}e^{-\frac{3}{e}-\frac{64m\varepsilon^2}{15dL^2}}.
\end{equation}
In the view of eqn. (\ref{N.epsilon}), from (\ref{p.either.bound}) and (\ref{p.or.bound}) for the probability measure $\rho_*$ such that $p=\rho_*^m(A_i^c)$ follows the result.
\end{proof}

\begin{thm}\label{err.lower.bound.p.para}
Under the same assumptions of Theorem \ref{err.lower.bound.p} for $\psi(t)=t^{1/2}\phi(t)$ and $\Psi(t)=t^{\frac{1}{2}+\frac{1}{2b}}\phi(t)$, the estimator $f_{\zz}$ corresponding to any learning algorithm converges to the regression function $f_\rho$ with the following lower rate:
$$\lim\limits_{\tau\to 0}\liminf\limits_{m\rightarrow\infty}\inf\limits_{l\in\mathcal{A}}\sup\limits_{\rho\in\mathcal{P}_{\phi,b}} Prob_{\zz}\left\{||f_\zz^l-f_\rho||_{\LL_{\nu}^2(X)}>\tau\psi\left(\Psi^{-1}(m^{-1/2})\right)\right\}=1,$$
where $\mathcal{A}$ denotes the set of all learning algorithms $l:\zz\to f_\zz^l$.
\end{thm}
\begin{proof}
Under the condition (\ref{m.bound}) from Theorem \ref{err.lower.bound.p} we get,
$$Prob_\zz\left\{||f_\zz-f_{\rho_*}||_{\LL_{\nu}^2(X)}> \frac{\varepsilon}{2}\right\}\geq\min\left\{\frac{1}{1+e^{-\ell_\varepsilon/24}},\vartheta e^{-\frac{1}{48}} e^{\left\{\frac{1}{48}\left(\frac{\al}{\psi^{-1}(\varepsilon/R)}\right)^{1/b}-\frac{64m\varepsilon^2}{15dL^2}\right\}}\right\}.$$
Choosing $\varepsilon_m=\tau R\psi(\Psi^{-1}(m^{-1/2}))$, we obtain
$$Prob_\zz\left\{||f_\zz-f_{\rho_*}||_{\LL_{\nu}^2(X)}> \tau \frac{R}{2}\psi(\Psi^{-1}(m^{-1/2}))\right\} \geq\min\left\{\frac{1}{1+e^{-\ell_\varepsilon/24}},\vartheta e^{-\frac{1}{48}}e^{c(\Psi^{-1}(m^{-1/2}))^{-1/b}}\right\},$$
where $c=\left(\frac{\al^{1/b}}{48}-\frac{64R^2\tau^2}{15dL^2}\right)>0$ for $0<\tau<\min\left(\frac{\sqrt{5d}L\al^{\frac{1}{2b}}}{32R},1\right)$.

Now as $m$ goes to $\infty$, $\varepsilon\to 0$ and $\ell_\varepsilon\to\infty.$ Therefore for $c>0$ we conclude that
$$\liminf\limits_{m\to\infty}\inf\limits_{l\in\mathcal{A}}\sup\limits_{\rho\in\mathcal{P}_{\phi,b}} Prob_\zz\left\{||f_\zz^l-f_\rho||_{\LL_{\nu}^2(X)}> \tau \frac{R}{2}\psi(\Psi^{-1}(m^{-1/2}))\right\}=1.$$
\end{proof}

Now we discuss the lower convergence rates in reproducing kernel Hilbert space norm.
\begin{thm}\label{err.lower.bound.k}
Let $\zz$ be i.i.d. samples drawn according to the probability measure $\rho\in\mathcal{P}_{\phi,b}$ under the hypothesis $dim(Y)=d<\infty$. Then for any learning algorithm ($\zz\to f_\zz\in\HH$) there exists a probability measure $\rho_*\in\mathcal{P}_{\phi,b}$ and $f_{\rho_*}\in \HH$ such that for all $0<\varepsilon<\varepsilon_o$, $f_\zz$ can be approximated as
$$Prob_{\zz}\left\{||f_{\zz}-f_{\rho_*}||_{\HH} >\varepsilon/2\right\} \geq \min\left\{\frac{1}{1+e^{-\ell_\varepsilon/24}}, \vartheta e^{\left(\frac{\ell_\varepsilon}{48}-\frac{cm\varepsilon^2}{\ell_\varepsilon^b}\right)}\right\}$$
where $\vartheta=e^{-3/e}$, $c=\frac{\beta 4^{b+2}}{15(b-1)dL^2}\left(1-\frac{1}{5^{b-1}}\right)$ and $\ell_\varepsilon\leq\frac{4}{5}\left(\frac{\al}{\phi^{-1}(\sqrt{5}\varepsilon/2R)}\right)^{1/b}$.
\end{thm}
\begin{proof}
The proof of the theorem follows the similar steps as of Theorem \ref{err.lower.bound.p}. From Proposition 6 \cite{Caponnetto}, for every $\ell=4\ell_1>16~(\ell_1\in\NN)$ there exists $N\in\NN$ and $\sigma_1\ldots,\sigma_N \in\{-1,+1\}^\ell$ such that
$$\sum\limits_{n=1}^{\ell}(\sigma_i^n-\sigma_j^n)^2\geq\ell, \text{    for all } 1 \leq i,j \leq N,~i \neq j$$
and
\begin{equation}\label{N.epsilon.k}
N\geq e^{\ell/24},
\end{equation}
where $\sigma_i=(\sigma_i^1,\ldots,\sigma_i^\ell).$

Now we take a fixed sequence $a\in\{-1,+1\}^{\ell_1}.$ Then for $\varsigma_i=(\varsigma_i^1,\ldots,\varsigma_i^{5\ell/4})=(a,\sigma_i)\in\{-1,+1\}^{5\ell/4}$ we have
\begin{equation}\label{varsigma.i.j}
\sum\limits_{n=1}^{5\ell/4}(\varsigma_i^n-\varsigma_j^n)^2\geq\ell, \text{    for all } 1 \leq i,j \leq N,~i \neq j.
\end{equation}
For $\varepsilon>0$, we consider
$$g_i=\sum\limits_{n=1}^{5\ell/4}\frac{\varepsilon\varsigma_i^n e_n}{\sqrt{\ell}\phi(t_n)}, \text{   for }i=1,\ldots,N,$$
where $t_n$'s are the eigenvalues of the integral operator $L_K$, $e_n$'s are the eigenvectors of the integral operator $L_K$ and the orthonormal basis of RKHS $\HH$.

The functions $f_i=\phi(L_K)g_i$ satisfy the source condition $f_i\in\Omega_{\phi,R}~(i=1,\ldots,N)$ for $\ell\leq\frac{4}{5}\left(\frac{\al}{\phi^{-1}(\sqrt{5}\varepsilon/2R)}\right)^{1/b}$. For $\ell>16$, we can choose $\varepsilon_o$ such that for every $0<\varepsilon<\varepsilon_o$ there exists a $\ell_\varepsilon$ satisfying this condition.

From eqn. (\ref{varsigma.i.j}) we get,
$$\varepsilon\leq||f_i-f_j||_\HH, \text{    for all } 1 \leq i,j \leq N,~i \neq j.$$
For $1\leq i,j \leq N_\varepsilon$, we have
$$||f_i-f_j||_{\LL_\nu^2(X)}^2\leq\sum_{n=1}^{5\ell_\varepsilon/4}\frac{\beta\varepsilon^2(\varsigma_i^n-\varsigma_j^n)}{\ell_\varepsilon n^b} \leq \sum\limits_{n=\frac{\ell_\varepsilon}{4}+1}^{5\ell_\varepsilon/4}\frac{4\beta\varepsilon^2}{\ell_\varepsilon n^b}\leq \frac{4\beta\varepsilon^2}{\ell_\varepsilon}\int_{\frac{\ell_\varepsilon}{4}}^{\frac{5\ell_\varepsilon}{4}}\frac{1}{x^b}dx =c'\frac{\varepsilon^2}{\ell_\varepsilon^b},$$
where $c'=\frac{\beta 4^b}{(b-1)}\left(1-\frac{1}{5^{b-1}}\right)$.

Then for the joint probability measures $\rho_{f_i}^m$, $\rho_{f_j}^m$ $(\rho_{f_i},\rho_{f_j}\in\mathcal{P}_{\phi,b},~1\leq i,j\leq N_\varepsilon)$ we get,
$$\mathcal{K}(\rho_{f_i}^m,\rho_{f_j}^m)=m\mathcal{K}(\rho_{f_i},\rho_{f_j})\leq \frac{16m}{15dL^2}||f_i-f_j||_{\LL_\nu^2(X)}^2\leq\frac{cm\varepsilon^2}{\ell_\varepsilon^b},$$
where $c=16c'/15dL^2$.

On applying Lemma 3.3 \cite{DeVore} with the probability measures $\rho_{f_i}^m$, $1\leq i\leq N_\varepsilon$ for the disjoint sets
$$A_i=\left\{\zz:||f_\zz-f_i||_\HH<\frac{\varepsilon}{2}\right\}, \text{   for  }1\leq i \leq N_\varepsilon,$$
we obtain
$$p:=\max\limits_{1\leq i\leq N_\varepsilon}\left(Prob\left\{\zz:||f_\zz-f_i||_\HH>\frac{\varepsilon}{2}\right\}\right) \geq\min\left\{\frac{N_\varepsilon}{N_\varepsilon+1},\sqrt{N_\varepsilon}e^{-\frac{3}{e}-\frac{cm\varepsilon^2}{\ell_\varepsilon^b}}\right\}.$$
From eqn. (\ref{N.epsilon.k}) for the probability measure $\rho_*$ such that $p=\rho_*^m(A_i^c)$ follows the result.
\end{proof}

Choosing $\varepsilon_m=\frac{2R\tau}{\sqrt{5}}\phi(\Psi^{-1}(m^{-1/2}))$ we get the following convergence rate from Theorem \ref{err.lower.bound.k}.
\begin{thm}\label{err.lower.bound.k.para}
Under the same assumptions of Theorem \ref{err.lower.bound.k} for $\Psi(t)=t^{\frac{1}{2}+\frac{1}{2b}}\phi(t)$, the estimator $f_{\zz}$ corresponding to any learning algorithm converges to the regression function $f_\rho$ with the following lower rate:
$$\lim\limits_{\tau\to 0}\liminf\limits_{m\rightarrow\infty}\inf\limits_{l\in\mathcal{A}}\sup\limits_{\rho\in\mathcal{P}_{\phi,b}} Prob_{\zz}\left\{||f_\zz^l-f_\rho||_\HH>\tau\phi\left(\Psi^{-1}( m^{-1/2})\right)\right\}=1.$$
\end{thm}

We obtain the following corollary as a consequence of Theorem \ref{err.lower.bound.p.para}, \ref{err.lower.bound.k.para}.
\begin{corollary}
For any learning algorithm under H\"{o}lder's source condition $f_\rho\in\Omega_{\phi,R},~\phi(t)=t^r$ and the polynomial decay condition (\ref{poly.decay}) for $b>1$, the lower convergence rates can be described as
$$\lim\limits_{\tau\to 0}\liminf\limits_{m\rightarrow\infty}\inf\limits_{l\in\mathcal{A}}\sup\limits_{\rho\in\mathcal{P}_{\phi,b}} Prob_{\zz}\left\{||f_\zz^l-f_\rho||_{\LL_\nu^2(X)}>\tau m^{-\frac{2br+b}{4br+2b+2}}\right\}=1$$
and
$$\lim\limits_{\tau\to 0}\liminf\limits_{m\rightarrow\infty}\inf\limits_{l\in\mathcal{A}}\sup\limits_{\rho\in\mathcal{P}_{\phi,b}} Prob_{\zz}\left\{||f_\zz^l-f_\rho||_\HH>\tau m^{-\frac{br}{2br+b+1}}\right\}=1.$$
\end{corollary}

If the minimax lower rate coincides with the upper convergence rate for $\la=\la_m$. Then the choice of parameter is said to be optimal. For the parameter choice $\la=\Psi^{-1}(m^{-1/2})$, Theorem \ref{err.upper.bound.p.para} and Theorem \ref{err.upper.bound.gen.p.para.choice} share the upper convergence rate with the lower convergence rate of Theorem \ref{err.lower.bound.p.para} in $\LL^2$-norm. For the same parameter choice, Theorem \ref{err.upper.bound.k.para} and Theorem \ref{err.upper.bound.gen.k.para} share the upper convergence rate with the lower convergence rate of Theorem \ref{err.lower.bound.k.para} in RKHS-norm. Therefore the choice of the parameter is optimal.

It is important to observe that we get the same convergence rates for $b=1$.
\subsection{Individual lower rates}
In this section, we discuss the individual minimax lower rates that describe the behavior of the error for the class of probability measure $\mathcal{P}_{\phi,b}$ as the sample size $m$ grows.
\begin{defi} A sequence of positive numbers $a_n~(n\in\NN)$ is called the individual lower rate of convergence for the class of probability measure $\mathcal{P}$,
if
$$\inf\limits_{l\in\mathcal{A}} \sup\limits_{\rho\in\mathcal{P}}\limsup\limits_{m\to\infty}\left(\frac{E_\zz\left(||f_\zz^l-f_\HH||^2\right)}{a_m}\right)> 0,$$
where $\mathcal{A}$ denotes the set of all learning algorithms $l:\zz\mapsto f_\zz^l$.
\end{defi}
In order to derive the individual lower rates we recall the following lemma (Lemma 3.2 \cite{Gyorfi}). The proofs of these lower rates are motivated from the ideas of the literature \cite{Caponnetto,Gyorfi}.
\begin{lemma}\label{ind.lemma}
Let $\Gamma\in\RR^m$, $s$ be a random variable taking values in $\{-1,+1\}$ with mean zero and $\eta=(\eta_i)_{i=1}^m$ consists of $m$ normally distributed independent random variables with mean zero and variance $\sigma^2$ which are independent of $s$. Suppose
\begin{equation}\label{y.gamma.eta}
\yy=s\Gamma+\eta.
\end{equation}
Then the error probability of the Bayes decision for $s$ based on $\yy$ is
$$\min\limits_{D:\RR^m\to\{-1,+1\}}Prob\{D(\yy)\neq s\}=\Phi\left(-\frac{||\Gamma||}{\sigma}\right),$$
where $\Phi$ is the standard normal distribution function.
\end{lemma}
\begin{thm}\label{err.ind.lower.bound.p}
Let $\zz$ be i.i.d. samples drawn according to the probability measure $\PP_{\phi,b}$ where $\phi$ is the index function satisfying the conditions that $\phi(t)/t^{r_1}$, $t^{r_2}/\phi(t)$ are nondecreasing functions and $dim(Y)=d<\infty$.  Then for every $\varepsilon>0$, the following lower bound holds:
$$\inf\limits_{l\in\mathcal{A}}\sup\limits_{\rho\in \PP_{\phi,b}}\limsup\limits_{m\to\infty} \left(\frac{E_\zz\left(||f_\zz^l-f_\HH||_{\LL_\nu^2(X)}^2\right)}{m^{-(bc_2+\varepsilon)/(bc_1+\varepsilon+1)}}\right)>0,$$
where $c_1=2r_1+1$ and $c_2=2r_2+1$.
\end{thm}
\begin{proof}
We consider the class of probability measures such that the target function $f_\HH$ is parameterized by ${\bf s}=(s_n)_{n=1}^\infty\in\{-1,+1\}^\infty$, the conditional probability measure $\rho(y|x)$ follows the normal distribution centered at $f_\HH$ and the marginal probability measure $\rho_X=\nu$. We derive the minimax lower rates over the considered class of probability measures which is the subset of $\PP_{\phi,b}$.

For given $\varepsilon>0$,
$$g=\sum\limits_{n=1}^\infty s_nR\sqrt{\frac{\varepsilon}{\varepsilon+1}\frac{\al}{n^bt_n}}\left(\frac{\phi(\al/n^b)}{\phi(t_n)}\right)n^{-(\varepsilon+1)/2}e_n,$$
where ${\bf s}=(s_n)_{n=1}^\infty\in\{-1,+1\}^\infty$, $t_n$'s are the eigenvalues of the integral operator $L_K$, $e_n$'s are the eigenvectors of the integral operator $L_K$ and the orthonormal basis of RKHS $\HH$. Then we have,
$$||g||_\HH^2\leq\sum\limits_{n=1}^\infty R^2\left(\frac{\varepsilon}{\varepsilon+1}\right)n^{-(\varepsilon+1)}\leq R^2.$$
Hence $f_\HH=\phi(L_K)g\in\Omega_{\phi,R}$.

The mean of the probability measure $f_\rho=f_\HH$ satisfies the moment condition (\ref{Y.leq.M.2}):
\begin{eqnarray*}
\int_Y\left(e^{||y-f_{\HH}(x)||_Y/M}-\frac{||y-f_{\HH}(x)||_Y}{M}-1\right)d\rho(y|x) &=& \left(\frac{1}{\sqrt{2\pi\sigma^2}}\right)^dS^d\int_0^\infty\left(e^{t/M}-\frac{t}{M}-1\right)e^{-t^2/2\sigma^2}t^{d-1}dt \\
&\leq& \frac{2\sigma^2S^d}{M^2\pi^{d/2}}\int_0^\infty e^{(\sqrt{2}\sigma t/M)}e^{-t^2}t^{d+1}dt \leq \frac{\Sigma^2}{2M^2},
\end{eqnarray*}
where the variance $\sigma^2:=\min\left(\frac{M^2}{2},\frac{\pi^{d/2}\Sigma^2}{4S^d\int_0^\infty e^{-t^2+t}t^{d+1}dt}\right)$ and $S^d$ is the surface of $d$-dimensional unit radius sphere.

Suppose
\begin{equation*}
y_i=f_\HH(x_i)+\eta_i=\sum\limits_{k=1}^\infty s_k\al_ke_k(x_i)+\eta_i,~~~i=1,\ldots,m,
\end{equation*}
where $\al_k=R\sqrt{\frac{\varepsilon}{\varepsilon+1}\frac{\al}{k^bt_k}}k^{-(\varepsilon+1)/2}\phi(\al/k^b)$ and $\eta_i$ follows the normal distribution centered at $0$. For $n\in\NN$, it follows
\begin{equation*}
\frac{\langle y_i,e_n(x_i)\rangle_Y}{||e_n(x_i)||_Y}=s_n\al_n||e_n(x_i)||_Y+\frac{\langle \eta_i,e_n(x_i)\rangle_Y}{||e_n(x_i)||_Y}+h_i,~~~i=1,\ldots,m,
\end{equation*}
where $h_i=\sum\limits_{k\neq n}s_k\al_k\langle e_k(x_i),e_n(x_i)\rangle_Y/||e_n(x_i)||_Y$.

The system of equations can be expressed in vector form as
\begin{equation}\label{y.gamma.eta.H}
\yy'=s_n\Gamma+\eta+H,
\end{equation}
where $\yy'=\left(\frac{\langle y_i,e_n(x_i)\rangle_Y}{||e_n(x_i)||_Y}\right)_{i=1}^m$, $\Gamma=\left(\al_n||e_n(x_i)||_Y\right)_{i=1}^m$, $\eta=\left(\frac{\langle \eta_i,e_n(x_i)\rangle_Y}{||e_n(x_i)||_Y}\right)_{i=1}^m$ and $H=(h_i)_{i=1}^m$.

We consider,
$$E_\xx\left(||f_{\zz}-f_\HH||_{\LL_\nu^2(X)}^2\right)=E_\xx\left(\sum\limits_{n=1}^\infty t_n\al_n^2\left(\frac{\langle f_\zz,e_n\rangle_\HH}{\al_n}-s_n\right)^2\right) \geq E_\xx\left(\sum\limits_{n=1}^\infty\frac{t_n\al_n^2}{4}\left(f_{sgn}\left(\frac{\langle f_\zz,e_n\rangle_\HH}{\al_n}\right)-s_n\right)^2\right),$$
where $f_{sgn}(x)$=$\left\{
                \begin{array}{ll}
                  1, & \hbox{$x\geq 0$;} \\
                  -1, & \hbox{$x<0$,}
                \end{array}
              \right.$
and $E_{\xx}(\xi)=\int_{X^m}\xi d\rho_X(x_1)\ldots d\rho_X(x_m)$.

Now we discuss the lower rates over the subclass of probability measure $\mathcal{P}_{\phi,b}$ parameterized by $\{-1,+1\}$-valued independent random variables $S=(S_n)_{n\in \NN}$ with mean zero.
\begin{eqnarray*}
E_SE_{\xx}\left(||f_\zz-f_\HH||_{\LL_\nu^2(X)}^2\right) &\geq& E_S\left(\sum\limits_{n=1}^\infty t_n\al_n^2 Prob\left(f_{sgn}\left(\frac{\langle f_{\zz},e_n\rangle_\HH}{\al_n}\right)\neq S_n\right)\right) \\
&=& E_S\left(E_\xx\left(\sum\limits_{n=1}^\infty t_n\al_n^2 Prob\left(f_{sgn}\left(\frac{\langle f_{\zz},e_n\rangle_\HH}{\al_n}\right)\neq S_n|\xx\right)\right)\right) \\
&\geq& E_S\left(E_{\xx}\left(\sum\limits_{n=1}^\infty t_n\al_n^2\min\limits_{D:\RR^m \to \{-1,+1\}}Prob(D(\yy')\neq S_n|\xx)\right)\right).
\end{eqnarray*}
We have the data structure (\ref{y.gamma.eta.H}) similar to Lemma \ref{ind.lemma}, except the presence of the vector $H$. Since the vector $H$ is independent of $S_n$. Therefore the probability of Bayes decision for $S_n$ based on $\yy'$ would be unaffected. Hence from Lemma \ref{ind.lemma} we get,
$$E_SE_\xx\left(||f_{\zz}-f_\HH||_{\LL_\nu^2(X)}^2\right) \geq E_S\left(E_\xx\left(\sum\limits_{n=1}^\infty t_n\al_n^2 \Phi\left(-\frac{||\Gamma||}{\sigma}\right)\right)\right),$$
where $||\Gamma||=\left(\sum\limits_{i=1}^m\al_n^2||e_n(x_i)||_Y^2\right)^{1/2}=\left(\sum\limits_{i=1}^m\al_n^2\langle K_{x_i}K_{x_i}^*e_n,e_n\rangle_\HH\right)^{1/2}$ and $\Phi$ is the standard normal distribution function with mean $0$ and variance $\sigma^2$.

Hence using Jensen's inequality we obtain
\begin{eqnarray*}
E_SE_{\xx}\left(||f_{\zz}-f_\HH||_{\LL_\nu^2(X)}^2\right) &\geq& E_S\left(\sum\limits_{n=1}^\infty t_n\al_n^2\Phi\left(-\frac{1}{\sigma}\left(\sum\limits_{i=1}^m\al_n^2E_\xx(\langle K_{x_i}K_{x_i}^*e_n,e_n\rangle_\HH)\right)^{1/2}\right)\right)   \\
&=&\sum\limits_{n=1}^\infty t_n\al_n^2\Phi\left(-\frac{1}{\sigma}\sqrt{m t_n\al_n^2}\right).
\end{eqnarray*}
Since $\phi(t)/t^{r_1}$, $t^{r_2}/\phi(t)$ are nondecreasing functions. It implies
$$t_n\al_n^2\geq un^{-(2br_2+b+\varepsilon+1)}\text{ and } t_n\al_n^2 \leq un^{-(2br_1+b+\varepsilon+1)},$$
where $u=R^2\left(\frac{\varepsilon}{\varepsilon+1}\right)\al\phi^2(\al)$.
Consequently, we obtain
$$E_SE_{\xx}\left(||f_{\zz}-f_\HH||_{\LL_\nu^2(X)}^2\right)\geq \sum\limits_{n=1}^\infty un^{-(2br_2+b+\varepsilon+1)}\Phi\left(-\frac{1}{\sigma}\sqrt{umn^{-(2br_1+b+\varepsilon+1)}}\right).$$
Therefore,
\begin{eqnarray*}
E_SE_\xx\left(||f_{\zz}-f_\HH||_{\LL_\nu^2(X)}^2\right) &\geq& \sum\limits_{n\geq m^{\left(\frac{1}{2br_1+b+\varepsilon+1}\right)}} un^{-(2br_2+b+\varepsilon+1)}\Phi\left(-\frac{1}{\sigma}\sqrt{u}\right)                                                                    \\
&\geq& u\Phi\left(-\frac{1}{\sigma}\sqrt{u}\right) \left\{\frac{1}{(2br_2+b+\varepsilon)}m^{-\left(\frac{2br_2+b+\varepsilon}{2br_1+b+\varepsilon+1}\right)} -m^{-\left(\frac{2br_2+b+\varepsilon+1}{2br_1+b+\varepsilon+1}\right)}\right\}.
\end{eqnarray*}
If $m\geq(4br_2+2b+2\varepsilon)^{2br_1+b+\varepsilon+1}$. Then it implies
$$E_SE_\xx\left(||f_{\zz}-f_\HH||_{\LL_\nu^2(X)}^2\right)\geq Cm^{-\left(\frac{2br_2+b+\varepsilon}{2br_1+b+\varepsilon+1}\right)}.$$
Suppose $c_1=2r_1+1$ and $c_2=2r_2+1$, then we get,
\begin{eqnarray*}
\inf\limits_{l\in\mathcal{A}}\sup\limits_{{\bf s}\in \{-1,+1\}^\infty}\limsup\limits_{m \to \infty}\left(\frac{E_\xx\left(||f_\zz^l-f_\HH||_{\LL_\nu^2(X)}^2\right)} {m^{-\left(\frac{bc_2+\varepsilon}{bc_1+\varepsilon+1}\right)}}\right) &\geq& C \inf\limits_{l\in\mathcal{A}}\sup\limits_{{\bf s}\in \{-1,+1\}^\infty}\limsup\limits_{m \to \infty}\left(\frac{E_\xx\left(||f_\zz^l-f_\HH||_{\LL_\nu^2(X)}^2\right)}{E_SE_\xx\left(||f_\zz^l-f_\HH||_{\LL_\nu^2(X)}^2\right)}\right) \\
&\geq& C\inf\limits_{l\in\mathcal{A}}E_S\left(\limsup_{m\to\infty}\frac{E_\xx\left(||f_\zz^l-f_\HH||_{\LL_\nu^2(X)}^2\right)} {E_SE_\xx\left(||f_\zz^l-f_\HH||_{\LL_\nu^2(X)}^2\right)}\right)                                                                           \\
&\geq& C>0.
\end{eqnarray*}
This follows as a consequence of Fatou's lemma. Then we achieve the desired result.
\end{proof}

\begin{thm}\label{err.ind.lower.bound.k}
Let $\zz$ be i.i.d. samples drawn according to the probability measure $\PP_{\phi,b}$ where $\phi$ is the index function satisfying the conditions that $\phi(t)/t^{r_1}$, $t^{r_2}/\phi(t)$ are nondecreasing functions and $dim(Y)=d<\infty$.  Then for every $\varepsilon>0$, the following lower bound holds:
$$\inf\limits_{l\in\mathcal{A}}\sup\limits_{\rho\in \PP_{\phi,b}}\limsup\limits_{m\to\infty} \left(\frac{E_\zz\left(||f_\zz^l-f_\HH||_{\HH}^2\right)}{m^{-(bc_2-b+\varepsilon)/(bc_1+\varepsilon+1)}}\right)>0.$$
\end{thm}

\end{document}